\documentclass[11pt]{article}

\usepackage{amssymb,amsmath,amsthm}
\usepackage{bbm}
\usepackage[noend]{algpseudocode}
\usepackage{graphicx}
\usepackage{verbatim}
\usepackage{url,xspace,hyperref}
\usepackage{caption}

\usepackage{subcaption}
\usepackage{multirow}
\usepackage{multicol}
\usepackage{latexsym}
\usepackage{amsmath,amssymb,enumerate}
\usepackage{algorithm,algorithmicx}
\usepackage{float}
\usepackage{xcolor}
\usepackage{mathrsfs}
\usepackage{cleveref}
\usepackage{booktabs}
\usepackage{breqn}

\usepackage{fullpage}
\usepackage{geometry}
\usepackage{setspace}
\usepackage{tcolorbox}
\usepackage{xspace}

\tcbuselibrary{skins,breakable}
\tcbset{enhanced jigsaw}

\usepackage{csquotes}
\makeatletter
\renewenvironment*{displayquote}
  {\begingroup\setlength{\leftmargini}{0.5cm}\csq@getcargs{\csq@bdquote{}{}}}
  {\csq@edquote\endgroup}
\makeatother

\newtheorem{theorem}{Theorem}[section]
\newtheorem*{theorem*}{Theorem}
\newtheorem{definition}{Definition}[section]

\newtheorem{lemma}[theorem]{Lemma}

\newtheorem{corollary}[theorem]{Corollary}

\newtheorem{fact}{Fact}[section]

\newcounter{note}[section]

\DeclareMathOperator*{\argmax}{arg\,max}
\DeclareMathOperator*{\argmin}{arg\,min}

\newcommand{\E}{\mathbb{E}} 
\newcommand{\hmu}{\hat{\mu}}

\newcommand{\vr}{{r}} 
\newcommand{\va}{{a}} 
\newcommand{\vp}{{p}} 
\newcommand{\vl}{{\ell}} 
 
\newcommand{\cS}{\mathcal{S}} 
\newcommand{\F}{\mathcal{F}} 
 
\newcommand{\KL}{\textsc{KL}}

\newcommand{\R}{\mathbb{R}}

\newcommand{\I}{\mathcal{I}}

\newcommand{\A}{\mathcal{A}}

\newcommand{\cP}{\mathcal{P}}
\newcommand{\C}{\mathcal{C}}

\newcommand{\D}{\mathcal{D}}

\newcommand{\Ot}{\widetilde{O}}

\newcommand{\BLO}{\mathtt{BLO}}
\newcommand{\decompose}{\mathtt{RFSM}}
\newcommand{\conv}{\mathtt{Conv}}
\newcommand{\bloact}{\mathtt{OSMDAct}}
\newcommand{\blofeed}{\mathtt{OSMDFeed}}

\newcommand{\osmd}{\mathtt{OSMD}}

\newcommand{\ignore}[1]{}

\newcommand{\paren}[1]{\ensuremath{\left(#1\right)}\xspace}

\newcommand{\bP}{\mathbb{P}}
\newcommand{\N}{\mathcal{N}}

\newcommand{\FP}{$\textsc{{Find-Permutation}}$}
\newcommand{\IP}{$\textsc{{Integral-Permutation}}$}

\renewcommand{\paren}[1]{{\left( #1 \right)}}

\newenvironment{tbox}{\begin{tcolorbox}[
		enlarge top by=5pt,
		enlarge bottom by=5pt,
		 breakable,
		 boxsep=0pt,
                  left=4pt,
                  right=4pt,
                  top=10pt,
                  arc=0pt,
                  boxrule=1pt,toprule=1pt,
                  colback=white
                  ]
	}
{\end{tcolorbox}}

\newcommand{\indic}[1]{\mathbb{I}\left[#1\right]}

\usepackage{natbib}
\usepackage{enumitem}

\newif\ifConfVersion

\title{Misalignment, Learning, and Ranking: Harnessing Users \\ Limited Attention}
\author{%
  Arpit Agarwal\thanks{The author is currently at FAIR, Meta. Work done while the author was at Columbia University.} \\
  Columbia University\\
  New York, NY 10027, USA \\
  \texttt{aa4931@columbia.edu }
  \and 
   Rad Niazadeh\\
  University of Chicago\\
  Chicago, IL 60637, USA \\
  \texttt{rad.niazadeh@chicagobooth.edu}
  \and 
    Prathamesh Patil \\
  University of Pennsylvania\\
  Philadelphia, PA 19104, USA \\
  \texttt{pprath@seas.upenn.edu }
}
\date{}

\begin{document}

\maketitle

\begin{abstract}
In domains like digital health or EdTech, recommendation systems often encounter a critical challenge: users, driven by impulsive or short-sighted tendencies, exhibit preferences that starkly contrast with the platform's forward-looking, long-term payoffs. This discrepancy complicates the task of learning-to-rank items to maximize the platform's total payoff, as it may lead to insufficient exploration on higher payoff items due to misalignment between user preferences and platform payoffs. Our paper addresses this challenge by leveraging the limited attention span of users, known as position bias. We consider a simple model in which a platform displays a ranked list of items in an online fashion to a sequence of $T$ arriving users. At each time, the arriving user selects an item by first considering a prefix window of a certain size of these ranked items and then picking the highest preferred item in that window (and the platform observes its payoff for this item as feedback). We study how to exploit the combinatorial structure of our model to design online learning algorithms that learn and optimize the unknown collected payoffs and obtain vanishing regret with respect to hindsight optimal benchmarks.

We first consider the bandit online learning problem with adversarial window sizes and stochastic i.i.d.\ payoffs. We design an active-elimination-based algorithm that achieves an  \emph{optimal instance-dependent} regret bound of $\mathcal{O}(\log (T))$, by showing matching regret upper and lower bounds. The key idea behind our result is using the combinatorial structure of the problem to either obtain a large payoff from each item or to explore by getting a sample from that item. To do so, our algorithm book-keeps a nested subset of items and picks the top item in each set to obtain a partial ordering at each time. It then updates this nested structure based on the number of samples received for each item and estimated confidence intervals.

Second, we consider the bandit online learning setting with adversarial payoffs and stochastic i.i.d. window sizes. In this setting, we first consider the full-information problem of finding the permutation that maximizes the expected payoff. By a novel combinatorial argument, we characterize the polytope of admissible item selection probabilities by a permutation and show it has a polynomial-size representation. Equipped with this succinct representation, we show how standard algorithms for adversarial online linear optimization in the space of admissible probabilities can be used to obtain polynomial-time online learning algorithm in the space of permutations with $\mathcal{O}(\sqrt{T})$ regret with respect to the best ranking in-hindsight.
\end{abstract}

\section{Introduction}
\label{sec:intro}

 Recommendation systems play a crucial role in various online platforms by suggesting relevant content, tasks, or options to users based on their preferences, skills, or previous interactions. Due to their effectiveness in helping platforms achieve a broad spectrum of goals, they are increasingly being used in a wide range of domains. These encompass digital content services such as music or video streaming platforms, crowd-sourcing applications such as campaign crowdfunding, and several other emerging applications, examples of which are course recommendations and peering tutors in EdTech,  suggesting exercise routines and dietary plans in wellness apps for digital health, and recommending tasks to volunteers within the nonprofit sector.  In almost all of these applications, users face an abundance of options; e.g. there are several podcasts that could be possibly suggested or several volunteer jobs to select from.  However, users typically have limited time and also limited attention from a psychological point of view to find their most preferred option. A critical aspect of all of these systems is then deciding how to display these options, usually in the form of a ranked vertical list, given the position bias of users towards higher ranked options.

To measure and optimize the effectiveness of the selected ranking, the recommendation system typically assigns payoffs for receiving clicks on items in the list. However, in several practical scenarios, there is a drastic \emph{misalignment} between platform payoffs and user preferences, where the latter is the driving force behind clicks. This misalignment often arises from multiple factors. On one hand, online platforms may pursue objectives that are longer-term or non-conventional, with complex attributions to immediate user engagement through clicks. On the other hand, users themselves sometimes exhibit short-sighted, impulsive, or irrational behaviors, overlooking their own long-term interests or genuine objectives.\footnote{There are strong evidences in both psychology and  behavioral economics literature that there are fundamental differences between users long-term benefits and immediate preferences, due to temporal discounting, impulsiveness, and present bias in users; e.g. see \cite{ainslie1975specious,frederick2002time,o1999doing}.}  For instance, users of a fitness app might gravitate toward less challenging workouts, whereas the platform prioritizes the user's overall health impact. Similarly, students using an EdTech app might favor shorter, more popular courses, while the platform focuses on fostering the user's long-term educational and learning objectives. In such contexts, user preferences and behavior are typically predictable through machine learning models (trained with a large amount of data on past user clicks) or due to the access of users to coordination devices such as public reviews.\footnote{Recommendation systems usually employ various methods to learn user preferences through both exploration and exploitation, while incentivizing exploration by diversifying recommendations, adding a random element to the ranking, or using various forms of discounts and incentives. See \cref{sec:related} for more context and related literature.} However, these misaligned payoffs are usually unknown to the platform due to their complex nature and being intertwined with longer-term goals. They can only be observed (or estimated) ex-post after the user interacts with the selected item --- hence the platform can only learn how to optimize the ranking through sequential experimentation and receiving feedback about these item payoffs over time after user interactions.

The above sequential experimentation and decision-making task can be cast as an \emph{online learning problem over the space of permutations}, where the platform selects a permutation over the items (i.e., arms) upon the arrival of each user and then obtains partial feedback on the payoff of the selected arm after user's click. However, in contrast to conventional online learning settings with partial feedback, e.g., stochastic multi-armed bandits~\citep{lai1985asymptotically,auer2002finite} or adversarial multi-armed bandits~\citep{auer2002nonstochastic,bubeck2012regret}, where the platform dictates which arm is selected, in recommendation systems the arms are being pulled by the users themselves. 
The biggest hurdle of the online learner who tries to maximize the platform's payoff over time is then the misalignment between the user preferences and the platform's payoffs, as the arm needed to be pulled for proper learning is not necessarily the user's most preferred arm. In fact, when the platform does not have ``dictatorship'' power in the selection of the arms, misaligned preferences of users can lead to insufficient exploration and therefore a lack of discovery of higher payoff arms. This phenomenon shows itself in the regret of the online learner with respect to the optimal solution in hindsight. The key challenge is to perform online learning, in particular exploration and exploitation, given this misalignment to obtain small (or vanishing) average regret.

In short, in this paper, we show how to mitigate the above challenge by harnessing the fact that users of the platform have limited attention. We study a simple learning-to-rank model where platform payoffs are unknown and user preferences are known, and we assume that they can be arbitrarily misaligned. 
Although in some applications the platform can try to incentivize users to pull its desired arm (e.g., by monetary discounts),  we take a different approach relying on the structure of the problem, particularly the fact that the platform displays a sorted list of the items and that users have limited attention (and hence position bias). We investigate whether the structure of the problem makes it amenable to exploration without the need for extra incentives, enabling the platform to perform learning and optimization without relying on excessive interventions. More formally, we ask the following research question:

\begin{displayquote}   
\emph{Do non-dictator online learning algorithms exist for ranking and displaying items to limited-attention users (under specific model primitives) that achieve the same asymptotic performance guarantee (i.e., up to sub-linear regret) as the optimal solution in hindsight?}
\end{displayquote}

The combinatorial aspect of the problem imposes extra information-theoretic and computational challenges on the decision-maker as she tries to explore and exploit. Importantly, even ignoring the misaligned preferences and platform payoffs, the action set of the decision maker, that is, the space of all permutations, is combinatorial. Despite that, as our main result in this work, we answer the above question in the affirmative by designing and analyzing polynomial-time online learning algorithms with optimal (asymptotic) regret guarantees both under stochastic and adversarial settings. 

\medskip
\noindent\textbf{Model.} More formally, we consider a finite-horizon sequential decision-making problem with partial feedback. Each time, the platform decides on a permutation over a finite set of items to display to the arriving user. Each item is associated with a time-dependent payoff, which is unknown to the platform when deciding a permutation to display. Given a permutation, the arriving user selects an item. Our users have ordinal preferences over items (although we use cardinal user utilities for simplicity of notation) and have limited attention: At each time, the user selects his most preferred item in a prefix of a specific size of the ranked list depending on his attention span, which we refer to as the user's attention window. Similar user behavior models have been studied in the literature of product ranking~\citep{ferreira2022learning,asadpour2022sequential,derakhshan2022product} or assortment optimization~\citep{aouad2021assortment,aouad2021display}. Users {have possibly heterogeneous preferences and attention spans across time}. 

In order to focus on studying the trade-off between the misalignment and learning, we make a simplifying assumption that the platform exactly knows the user preferences or the ordering over the user's utility of selecting each item before decision-making at each time (or equivalently, we consider a deterministic utility choice model).
For example, these utilities might be known to the platform through public/private signals such as user reviews, learned preferences, etc.
\footnote{In this simple model, given the window size, the item selection is deterministic --- and still our model bears several technical challenges. We leave studying the more complex model where the user's selection is not deterministic and is drawn from a known distribution (for example, when the user has a randomized utility choice model) for future work.} {\color{black}(We further extend our results to the setting where user preferences should be learned first through collecting reviews and social learning. See \Cref{sec:practical-social-learning})}.

The platform's goal is to adaptively select permutations given the preferences of the arriving users to maximize her expected total payoff in $T\in\mathbb{N}$ rounds of decision-making. In line with both stochastic and adversarial bandit models in the literature, we study two basic, yet fundamental, settings of our problem:
\begin{enumerate}[label=(\roman*)]
    \item \Cref{sec:stochastic}: stochastic i.i.d. item payoffs drawn from a \emph{unknown} distribution (and possibly adversarial attention windows). 
    \item \Cref{sec:adversarial}: adversarial item payoffs (and i.i.d.\ attention window sizes drawn from a \emph{known} distribution.)
\end{enumerate}

It is essential to highlight that the algorithm does not know the user's window size in the first setting. Similarly, in the second setting, the algorithm does not know the exact realization of window sizes but knows the underlying distribution. Given these primitives, our benchmark in setting (i), denoted by $\textsc{OPT}_{\textrm{stochastic}}$, is the clairvoyant optimal in-hindsight algorithm that knows the mean payoffs of each item and the adversarial choices of window sizes, and given this information finds the best-fixed mapping of the user utilities to permutations. In setting (ii), our benchmark $\textsc{OPT}_{\textrm{adversarial}}$ is a similar clairvoyant optimal in-hindsight algorithm that, given the adversarial choices of payoffs and window size distribution, finds the best-fixed mapping of the user utilities to permutations. For simplicity of technical exposition, we first restrict our attention to fixed utilities over time in each of the Sections~\ref{sec:stochastic} and~\ref{sec:adversarial}. For this particular case, our benchmarks correspond to a fixed permutation; Later, we explain how our results can easily extend to the case with changing (but known) utilities.   

\medskip
\noindent\textbf{Main Technical Results.}  We first consider the case of stochastic payoffs with adversarial window sizes (setting (i)). As our main result in this setting, we provide a \emph{complete characterization} of the achievable instance-dependent (asymptotic) regret with respect to our benchmark $\textsc{OPT}_{\textrm{stochastic}}$. Our regret characterization is two-fold:

\smallskip
\noindent~(a) First, we present a computationally efficient online learning algorithm that achieves an instance-dependent $O(\log T)$ regret bound. The form of our regret bound is tightly related to the combinatorial nature of our problem, i.e., selecting permutations, and hence depends on the pairwise gaps between the item payoffs. Specifically, it has the following form: $O \left( \sum_{i\in[2:n]} \frac{\log (nT)}{ \Delta_{i-1 i}}\right)~,$

where $n$ is the number of items, and for any pair $i,j\in [n]$ of items, $\Delta_{ij}$ is the gap in mean payoffs (the items are assumed to be labeled in decreasing order of mean payoffs). See \Cref{thm:stochastic_main_general} for more details on the above instance-dependent regret bound under changing utilities. See also 
\Cref{thm:stochastic_main} for an improved instance-dependent regret bound with fixed utilities. Surprisingly, our optimal regret algorithm achieving both of these bounds, \Cref{alg:stochastic}, is relatively simple and interpretable.

\smallskip
\noindent~(b) We then complement the above results with matching information-theoretic instance-dependent lower-bounds by showing that the bound in \Cref{thm:stochastic_main} is the best regret achievable under the worst-case choice of payoff distributions and window sizes (even if the window sizes are known to the algorithm upfront) when utilities are fixed. See \Cref{thm:lb}. A corollary of this lower-bound theorem is the tightness of the bound \Cref{thm:stochastic_main_general}, in the particular case when utilities are fixed and payoffs are in reverse order of utilities.

Switching to the setting with adversarial payoffs and stochastic window sizes (setting (ii)), we study a different style of algorithms that try to learn and optimize in the lower-dimensional space of marginal selection probabilities of each item --- instead of optimizing in the higher-dimensional space of permutations. Our main result is a computationally efficient online learning algorithm, \Cref{alg:blo}, that achieves the (information-theoretically) optimal adversarial regret bound of $O(\sqrt{nT})$ with respect to the benchmark  $\textsc{OPT}_{\textrm{adversarial}}$ --- both when utilities are fixed or changing over time. Again our algorithm is surprisingly natural and heavily uses the combinatorial structure of the problem in order to reduce the problem of learning permutations to the problem of learning marginal selection probabilities through a polynomial-time reduction.

We also provide results for two extensions to our main setting that are based on practical considerations: (1) unknown utilities, and (2) delayed payoff feedback (see Appendix~\ref{app:extensions} for details).


\medskip
\noindent\textbf{High-level Overview of Techniques.} \Cref{alg:stochastic} for the stochastic payoff setting builds upon the idea of active-arm elimination for classical bandits, where one maintains a set of \emph{active items} which are played in a round-robin fashion until we gather sufficient evidence indicating the suboptimality of some item. At this point, the suboptimal item is permanently eliminated from the active set. The regret incurred from playing any item can then be bounded by bounding the number of times it is played before some superior item can certify its suboptimality with confidence. However, there is an obvious barrier to using this idea directly in our model -- we can neither force a particular item to be selected nor can we prevent a particular item from being selected (i.e., eliminate), as the selected item strongly depends upon the realized window length in that trial which is beyond our control. Consequently, the number of selections of each item in our set can differ arbitrarily, making the regret analysis challenging; all analyses for classical bandit algorithms implicitly depend on controlling the number of times any sub-optimal item is played across the time horizon~\citep{lai1985asymptotically,auer2002finite,lattimore2020bandit,slivkins2019introduction}.

The key idea behind our algorithm is the notion of \emph{inversion}, which is a novel charging argument that allows us to directly bound the regret due to selecting sub-optimal items without first having to control the number of selections. If the estimated means are a good approximation to the actual means, and the confidence intervals are non-overlapping, then the permutation produced will be optimal. If not, we will misposition any item only a finite number of times (incurring bounded regret in the process). Breaking ties in favor of items selected fewer times is crucial to achieving this guarantee. To hold to these primitives, our algorithm implicitly and recursively book-keeps a nested subset of items based on the number of samples received for each item and estimated confidence intervals.: It computes the top set (empirically unbeaten items) and puts the least played item at the end of the partial permutation. It then recomputes the top set (since now, the items that were only beaten by this item that was just added to the permutation have now become top items). A combination of the mentioned charging argument and controlling confidence bounds help with analyzing this algorithm. We elaborate on this intuition in \Cref{sec:stochastic-analysis}, where we analyze the regret achieved by our algorithm. 

We take a slightly different approach to solving the problem with adversarial payoffs. Fixing the particular user behavior we consider in our model, each permutation of the products induces a specific vector of \emph{marginal selection probabilities} on the items. The ``right'' choice of this marginal selection probability vector can help the platform learn and optimize by doing exploration and exploitation. However, not all marginal selection probabilities can be chosen, as the platform only selects a permutation (possibly randomized). At first glance, the space of permutations is exponential-size; hence, the space of \emph{admissible} marginal selection probabilities might have had an exponential-size representation. We ask two immediate questions: given this combinatorial construct, can we induce certain distributions for item selections, e.g., distributions that allow the platform to explore? What about inducing even a better distribution, that is, the output of an online linear optimization algorithm over marginal selection probabilities?

First, as a warm-up result, we construct a $O(n)$-support distribution over permutations that induces a uniform distribution over marginal selection probabilities, i.e.\ each item is selected with probability $1/n$, 
under a special case when users are \emph{lazy}.
The laziness assumption requires that likelihood of a shorter attention window size is higher for each user. 
Using this construction one can use the $\epsilon$-greedy
algorithm in order to achieve a $O(T^{2/3})$ regret bound.

Next, we study the polytope of admissible marginal selection probabilities.   We use a combinatorial recursive argument to find a closed-form and succinct, i.e., polynomial-size,  representation for this polytope. Our representation comes hand in hand with a polynomial-time (exact) rounding algorithm that, given a marginal selection probability matrix in our problem, decomposes that into a convex combination of integral permutation. We then show how to use this algorithm in a black-box fashion to reduce our problem to the bandit version of the adversarial online linear optimization~\citep{bubeck2012regret,audibert2009minimax,auer1995gambling}. See more details in \Cref{sec:adversarial}, in particular, \Cref{alg:rounding} for the decomposition and \Cref{alg:round-integral} that helps with the recursive step.

\subsection{Further Related Work}

\label{sec:related}

Besides the rich lines of work on stochastic, and adversarial bandits, our work is connected to several other streams of research:

\smallskip
\noindent\emph{Incentivizing exploration.} The problem of utility-driven selection of items by users on a platform that lacks dictatorship is also studied in the \emph{incentivizing exploration} literature, where the aim is to persuade rational users to pick arms that are aligned with the learning objectives of the platform (both exploration and exploitation). This is typically achieved either through various forms of discounts and incentives, e.g. Bayesian Incentive Compatible (BIC) bandit algorithms without money  \citep{kremer2014implementing,mansour2016bayesian,immorlica2019bayesian,mansour2020bayesian} and with money  \citep{frazier2014incentivizing,wang2018multi}, or by considering explicit types of platform intervention such as hiding information from the users  \citep{immorlica2020incentivizing,papanastasiou2018crowdsourcing}. While these approaches partially mitigate the issue of misaligned preferences, they usually come at a loss in performance due to the platform's interventions~\citep{sellke2021price} or rely on occasionally strong game-theoretic assumptions that may not hold in practice --- such as the power of commitment from the side of the platform and assuming users know the recommender mechanism, or restricting to BIC solution concept and assuming users are Bayesian with a common prior. Moreover, these works implicitly assume a particular form of misalignment between user utilities (empirical mean rewards of items), and platform payoffs (true mean rewards of items). Our work on the other hand aims to understand whether the structure of the online learning problem itself can be leveraged to nudge user behavior in a favorable direction without \emph{any external intervention}, even when the user utilities are arbitrarily misaligned with the platforms objectives.    

\smallskip
\noindent\emph{Learning to rank.} Online learning to rank is a core problem in information retrieval, machine learning, and online platforms.
Many provably efficient algorithms have been proposed for this problem in specific click models, under different feedback structures, and with different objectives. The canonical examples of click models are the well-studied Cascade Model~\citep{kveton2015cascading,zong2016cascading}, Dependent Click Model (with single or multiple clicks)~\citep{katariya2016dcm}, Position-Based Model with
known position bias~\citep{lagree2016multiple}, Window-based Consider-then-Choose users~\citep{ferreira2022learning}, and general click model~\citep{zoghi2017online,lattimore2018toprank}. There are also several works on adversarial models for learning to rank. These settings are the most
similar to the stochastic position-based and document-based
models, but with the additional robustness bought by the
adversarial framework. For example, \cite{radlinski2008learning}
and the follow up work in \cite{slivkins2013ranked} where the
learner’s objective is to maximize the probability that a user
clicks on any item rather than rewarding multiple clicks.
See the survey by \cite{hofmann2011probabilistic} for more details.

There are several differences between the techniques used in LTR literature and the techniques used in our work. Firstly, most algorithms in the LTR literature rely on learning the optimal ranking between items, typically based on item click-through rates (CTRs). This is aligned with the objective of maximizing the number of clicks or minimizing regret. In our setting, however, the objective of learning a ranking (based on payoffs) might be misaligned with the objective of regret minimization and might even be infeasible. 
Secondly, an important challenge in the design of our algorithm is to bound the regret due to sub-optimal permutations resulting from errors in estimating payoffs. How can we translate the confidence radius for payoffs of individual items into bounds on regret incurred by the permutation? 
We overcome this challenge by using a charging argument that allows us to charge the regret of the permutation to different items based on the realized attention window. This allows us to balance the number of times an item is selected at a given position with regret due to this item.

\smallskip
\noindent\emph{Product ranking in online platforms.} There is also a rich literature on product ranking optimization in the revenue management literature~\citep{derakhshan2022product,abeliuk2016assortment,gallego2020approximation,aouad2021display,asadpour2022sequential,sumida2021revenue,golrezaei2021learning}, which focus on maximizing revenue (i.e., product-weighted objectives), or a combination of revenue and social welfare, or similar other economically justified objectives --- in different contexts with different applications and agendas.

\smallskip
\noindent\emph{Online combinatorial learning}.
Moving beyond online learning to only rank, our work contributes to the rich literature on adversarial oracle efficient online with exact  algorithms~\citep{kalai2005efficient,dudik2020oracle} or approximation algorithms~\citep{kakade2007playing,niazadeh2022online} as the oracle. Our technical approach, especially in the adversarial setting, is also related to online combinatorial optimization~\citep{audibert2014regret}.

\smallskip
\noindent\emph{Online contextual learning}. In this paper, for simplicity, we consider item-dependent payoffs; however, in recommendation system applications where personalization is allowed, the platform can also benefit from user-dependent payoffs through \emph{contextual learning}~\citep{li2010contextual,li2018online,besbes2021contextual,cohen2020feature,leme2018contextual,liu2018contextual,gollapudi2021contextual}. Fitting our paper into this framework, one can think of our approach as a form of contextual learning where user preferences are essentially the context and used by the platform to rank the products; notably, this context changes the way an item is selected in our model, but not the platform payoffs upon selection (and hence no cross-learning is happening nor required).\footnote{Considering a fully contextual-model with context-dependent payoffs, and then trying to benefit from cross-learning across the products, is an interesting avenue for future research.}

\section{Preliminaries}
We consider a setting where a platform displays 
items to each arriving user over a finite time-horizon $T$.
Specifically, at each time $t\in [T]$, the platform 
displays a permutation $\pi^t$ over $n$ items 
to the user.  
The utility of the user for item $i \in [n]$
at time $t$ is denoted by $u_i^t$.
We assume the utilities at time $t$ are known
to the platform before displaying $\pi^t$.
For the ease of technical exposition, when needed, we focus on a setting where utilities are fixed over time and later show how our results extend to the general setting of changing utilities.
The user has `limited-attention' in its selection-- 
it first selects a window of size $w^t \in [n]$
and then selects the highest utility item 
that falls in the window of size $w^t$, i.e.\
the selected item $y^t = \argmax_{a \in [w^t]} u_{\pi^t(a)}$, 
where $\pi(a)$ denotes the item in position $a$
in $\pi$.
We assume that the realization of $w^t$ does not depend on the permutation $\pi^t$
that is played at time $t$.
The payoff of item $i$ at time $t$ is denoted by  $r_i^t$.
We consider the setting of stochastic payoffs in \Cref{sec:stochastic} and adversarial payoffs in \Cref{sec:adversarial}.
The consider a limited feedback model for payoffs: the platform only observes the selected item $y^t$ and its payoff $r_{y^t}^t$.
Note that the platform does not observe
the realized window $w^t$.
The goal of the platform is to minimize its regret 
against a suitable benchmark.
In the remainder of the paper we will study different 
settings of payoffs and window realizations,
and define regret benchmarks appropriately.


\section{Stochastic Payoffs and Adversarial Attention Windows}
\label{sec:stochastic}
In this section, we consider the setting where the item payoffs are stochastic, and our objective is to be competitive against the best sequence of permutations over items, given an arbitrary (potentially adversarial and adaptive) sequence of item utility orderings and attention window realizations. 

\subsection{Model and Notations}
We assume that for each item $i\in [n]$, there is an underlying (stationary) Gaussian\footnote{We assume Gaussianity with unit variance only for convenience. All our results can be extended quite straightforwardly to general sub-Gaussian distributions.} distribution $\mathcal{N}(\mu_i,1)$ over payoffs. In any trial $t\in [T]$, in the event that item $i\in [n]$ is selected, we receive a payoff $r^t_i$ sampled i.i.d. from its payoff distribution. Note that an item $i\in [n]$ is selected in trial $t$ iff it is the highest utility item in the prefix of length $w^t$, the realized attention window length in that trial, in the permutation $\pi^t$ displayed in that trial. Consequently, the (pseudo) regret $R(\pi^t;w^t)$ incurred due to this selection is $\Delta_{y^t_*y^t}:=\mu_{y^t_*} - \mu_{y^t}$, where $y^t = \argmax_{i:\pi^t(i) \leq w^t} u^t_i$ is the item selected from the displayed permutation $\pi^t$, and $y^t_* = \argmax_{i:\pi^{t}_*(i)\leq w^t} u^t_i$ is the item that would have been selected from the best permutation $\pi^{t}_*$ in hindsight for that trial. The objective then is to minimize this cumulative (pseudo) regret over all trials $t\in [T]$ over an unknown time horizon $T$:
$$
\textsc{Regret}(T)\triangleq\sum_{t=1}^{T}R(\pi^t;w^t)
$$
We make no assumptions on the sequence of the sequence of attention window realizations $\{w^t\}_{t\in [T]}$, nor the sequence of item utility orderings $\{u^t_i\}_{t\in [T]}$ except that the latter is observable prior to decision making in each trial; in particular, both can be chosen by an adaptive adversary depending on the past payoffs, as well as the entire execution history of our algorithm.

\subsection{The Upper Bound}

\subsubsection{The Algorithm}

We start by providing more details around our algorithm. 
As with most bandit algorithms, we maintain two statistics per item $i\in [n]$ -- the cumulative payoff observed $r_i$, and the number of times it was selected $N_i$. These statistics are then used to compute an estimate of the mean payoff $\hmu_i = r_i/N_i$, as well as a confidence interval $c_i = \Theta(\sqrt{\log T/N_i})$, which gives us high-probability lower and upper bounds on the true mean payoff $\mu_i \in [\hmu_i - c_i,\hmu_i+c_i]$ of that item. In each trial $t\in [T]$, we use these estimates to greedily construct a permutation $\pi^t$ to display in the following manner: starting with a partial permutation which is initially empty, we identify a set $S\subseteq [n]$ of empirically undominated items from the set of items $A$ that are not already in the partial permutation. These are items $i\in A$ that are not beaten by any other item $j\in A$ with confidence, i.e. $\hmu_i + c_i > \hmu_j - c_j$. From this set $S$, we select the item $i^*\in S$ with the least number of selections (breaking ties arbitrarily) and append it to the partial permutation constructed thus far, followed by any other items that have lower utility than $u^t_{i^*}$ that are not already in the partial permutation. The order of these lower utility items does not matter since they will never be selected under any attention window realization (they are preceded by $i^*$, which is a higher utility item). We then remove these items that have now been added to the partial permutation ($i^*$ and all items with lower utility than $u_{i^*}$) from the set $A$, and recurse, terminating when all items have been assigned to a position in $\pi^t$ (or equivalently, $A$ becomes empty). A formal description is in Algorithm ~\ref{alg:stochastic}, with the construction of the permutation in any round being defined in Algorithm ~\ref{alg:find_permutation}.

The following theorem then gives a high probability bound on the cumulative regret achieved by our algorithm for the general case with changing utilities. 

\begin{theorem}
\label{thm:stochastic_main_general}
Given items $[n]$ with Gaussian payoff distributions $\{\mathcal{N}(\mu_i,1)\}_{i \in [n]}$, let $\{w^t\}_{t>0}$ and $\{u^t_i\}_{i\in [n],t>0}$ be any arbitrary sequence of attention window lengths and utility orderings, respectively, possibly adversarial and adaptively chosen. Then with probability at least $1-\delta$, where $0<\delta\leq 1$ is any specified confidence parameter, the cumulative regret of Algorithm~\ref{alg:stochastic} over any time horizon $T$ is bounded by  

\[
 \sum_{t=1}^T R(\pi^t;w^t) = O \left( \sum_{i = 2}^{n} \frac{\log (nT/\delta)}{ \Delta_{i-1 i}}\right) 
    \,,
\]
where for any pair $i,j\in [n]$ of items, $\Delta_{ij} = \mu_i - \mu_j$ is the gap in expected payoffs, and items are assumed to be labeled in decreasing order of mean payoffs.
\end{theorem}

The above high probability regret bound also gives us an upper bound on the expected regret incurred by our algorithm as a corollary (with expectation over the random payoffs observed by our algorithm). We have that
\begin{align*}
    \E\left(\sum_{t=1}^T R(\pi^t;w^t)\right) \leq C\cdot (1-\delta)\cdot\left( \sum_{i = 2}^{n}  \frac{\log (nT/\delta)}{ \Delta_{i-1i}}\right) + \delta\cdot \Delta_{\max}\cdot T
    =O \left( \sum_{i = 2}^{n} \frac{\log (nT/\Delta_{\max})}{ \Delta_{i-1i}}\right),
\end{align*}
where $C$ is the constant in the regret upper bound in Theorem~\ref{thm:stochastic_main}, and $\Delta_{\max} := \max_{i,j\in [n]}\mu_i-\mu_j$ is the maximum (pseudo) regret that can be incurred in any trial. The final bound follows by setting $\delta = (\Delta_{\max}T)^{-1}$, using the usual doubling trick for unknown time horizons.

In the special case that the item utilities are \emph{stationary}, our algorithm achieves potentially an even stronger regret upper bound, depending on the underlying structure of the utility parameters of the instance. Furthermore, as we shall see in the following section, this regret bound is also \emph{tight}, i.e. for any fixed stationary utility ordering, there exists a sequence of attention window realizations where any algorithm must incur the regret we outline below. Before presenting this stronger regret bound, we shall introduce some notation. Given an instance over $[n]$ items, with utilities $\{u_i\}_{i\in [n]}$ and mean payoffs $\{\mu_i\}_{i\in [n]}$, we say an item $j\in [n]$ is \emph{dominated} if there exists some item $i\in [n]$ that beats item $j$ both in terms of utility, as well as mean payoff, i.e. $u_i>u_j$ and $\mu_i>\mu_j$. We then define $D := \{j:\exists i\in [n] \text{ where } u_i > u_j \text{ and } \mu_i >\mu_j\}$ to be the set of all dominated items, with $S := [n]\setminus D := \{s_1,\ldots,s_{n-|D|}\}$ being the set of all remaining (undominated) items labeled in decreasing order of mean payoffs, i.e. $\mu_{s_i}>\mu_{s_j}$ for $i<j$. Furthermore, for any undominated item $s_i\in S$, we define set $D_{s_i}:= \{j\in D: s_i = \argmax_{i'\in S: u_{i'}>u_j} \mu_{i'}\}$ to be the set of items dominated by $s_i$ that are not already dominated by some other item $i'$ with a larger mean payoff than $s_i$. Note that the sets $\{D_{s_i}\}_{s_i\in S}$ are a disjoint partition of the dominated items $D$.

The following theorem then characterizes the regret achieved by Algorithm~\ref{alg:stochastic} for the special case of stationary utilities.
\begin{theorem}
\label{thm:stochastic_main}
Given items $[n]$ with stationary utilities $\{u_i\}_{i\in [n]}$ and Gaussian payoff distributions $\{\mathcal{N}(\mu_i,1)\}_{i \in [n]}$, let $\{w^t\}_{t>0}$ be any arbitrary sequence of attention window lengths, possibly adversarial and adaptively chosen. Then with probability at least $1-\delta$, where $0<\delta\leq 1$ is any specified confidence parameter, the cumulative regret of Algorithm~\ref{alg:stochastic} over any time horizon $T$ is bounded by  

\[
 \sum_{t=1}^T R(\pi^t;w^t) = O \left( \sum_{i = 2}^{|S|}  \frac{\log (nT/\delta)}{ \Delta_{s_{i-1}s_{i}}} + \sum_{i = 1}^{|S|} \sum_{j \in D_{s_i}} \frac{\log (nT/\delta)}{ \Delta_{s_ij}}\right)
    \,.
\]

\end{theorem}

\begin{algorithm}[t]
\caption{\textsc{Algorithm for Stochastic Payoffs}}
\label{alg:stochastic}
\begin{algorithmic}[1]
\State \textbf{Input:} items $[n]$, and confidence $\delta$..
\State $\forall \text{ items } i\in [n],\text{set cumulative payoff } r_i\leftarrow 0,\text{ number of selections }  ~N_i \leftarrow 0$
\For {t= 1,\ldots}
\State{Observe utilities $\{u^t_i\}_{i \in [n]}$}
\State $\pi^t\leftarrow \FP([n],t,\delta,\{r_i\}_{i\in [n]},\{N_i\}_{i\in [n]},\{u^t_i\}_{i\in [n]})$
\State Display $\pi^t$; observe selection $y^t$, payoff $r^t_{y^t}$.
\State Update $r_{y^t} \leftarrow r_{y^t} + r^t_{y^t}$, $N_{y^t}\leftarrow N_{y^t} + 1$.
\EndFor
\end{algorithmic}
\end{algorithm}

\begin{algorithm}[t]
\caption{\textsc{Find-Permutation}}
\label{alg:find_permutation}
\begin{algorithmic}[1]
\State \textbf{Input:} items $[n]$, trial $t$, confidence $\delta$, cumulative payoff $\{r_i\}_{i\in [n]}$, number of selections $\{N_i\}_{i \in [n]}$, utilities $\{u_i\}_{i \in [n]}$
\State $\forall$ items $i\in [n]$, set estimated mean payoff $\hmu_i\leftarrow r_i/N_i$, confidence $c_i \leftarrow \sqrt{\log (4nt^2/\delta)/N_i}$.
\State Remaining items $A\leftarrow [n]$, partial permutation $\pi \leftarrow \emptyset$
\While{$A \neq \emptyset$}
\State $S\leftarrow \left\{i\in A: \forall j\in A,~\hmu_i + c_i > \hmu_j- c_j\right\}$ \Comment{Empirically undominated arms}
\State $i^* \leftarrow \argmin_{i\in S} N_i$ \Comment{Least played empirically undominated arm}
\State $B_{i^*}\leftarrow \{j\in A: u_j < u_{i^*}\}$ \Comment{Arms blocked by $i^*$}
\State Update partial permutation $\pi \leftarrow \pi + i^* + B_{i^*}$ \Comment{$B_{i^*}$ ordered arbitrarily}
\State $A\leftarrow A\setminus (i^* \cup B_{i^*})$
\EndWhile
\State \textbf{Return:} $\pi$
\end{algorithmic}
\end{algorithm}

\subsubsection{Regret Analysis}
\label{sec:stochastic-analysis}
The proof of the regret upper bound, at its core, is a decomposition that charges the pseudo regret incurred by displaying a permutation in a trial to a particular pair of items -- one being the  item selected from the displayed permutation under the realized attention window in that trial, and other being the item that would have been selected from an optimal permutation under the same attention window realization. The challenge then is to show that the total amount of regret that can be charged to any pair is then bounded (roughly by the inverse of the squared gap between the mean payoffs of the items in that pair). The key to this charging argument is establishing the structure of an optimal permutation, which we describe in the following lemma.
\begin{lemma}
\label{lem:opt}
    Given an instance over $[n]$ items with utilities $\{u_i\}_{i\in [n]}$, and mean payoffs $\{\mu_i\}_{i\in [n]}$, let $S=\{s_1,s_2,\ldots, s_k\}$ be the set of undominated items labeled in decreasing order of mean payoffs (or equivalently, in increasing order of utilities), i.e. for any $ s_l,s_h\in S$ such that $l<h$, then $\mu_{s_l}>\mu_{s_h}$ (equivalently, $u_{s_l}<u_{s_h}$). Then any permutation that maximizes the expected payoff for any attention window length $1\leq w\leq n$ chosen independent of the realized payoffs of the arms, belongs to the family 
    \[\Pi^* := \{(s_1,B_{s_1},s_2,B_{s_2},\ldots,s_k,B_{s_k})\},\]
    where for any $s\in S$, $B_{s}$ is a permutation of the set $D_{s}$.
\end{lemma}
\begin{proof}
We begin by bounding the maximum expected payoff that can be obtained by any permutation for a fixed attention window length $1\leq w\leq n$. Given utilities $\{u_i\}_{i\in [n]}$, for each item $i\in [n]$, we define the set $B_i:=\{j\in [n]: u_i>u_j\}$ to be the set of items with utilities at most that of item $i$. Now we claim that for any attention window of length $w$, the maximum expected payoff of any permutation is at most $\max_{i\in [n]: |B_i| \geq w-1} \mu_i$. This follows from the simple observation that no item $j$ with $|B_j| < w-1$ can be selected in a attention window of length $w$, regardless of it's position in the permutation, as the attention window must necessarily contain some higher utility (more preferred) item $j'$ with $u_{j'}>u_j$ due to the pigeonhole principle. Given this observation, the optimality of $\Pi^*$ follows immediately, as any permutation from this family achieves expected payoff that matches this upper bound for \emph{all} attention window lengths $1\leq w\leq n$.  
\end{proof}

Before proceeding with the remainder of the proof, observe that the above described family of optimal permutations contains the permutation that will be returned by the greedy Algorithm~\ref{alg:find_permutation} when the estimated means of the arms respect the ordering induced by the true means, and the confidence intervals are non-overlapping, i.e. $\{\hmu_i,c_i\}_{i\in [n]}$such that for all pairs of items $i,j\in [n]$ where $\mu_i > \mu_j$, $\hmu_i - c_i > \hmu_j + c_j$.

We now consider the following event of interest, which intuitively guarantees that across the entire execution of our algorithm, the estimated means for all arms will not differ significantly from their corresponding true means. Moreover, we show that this good event will occur with a sufficiently high probability $1-\delta$, where $\delta$ is the given confidence parameter.   

\begin{lemma}
\label{lemm:good-event}
Given items $[n]$ with utilities $\{u_i\}_{i\in [n]}$, and Gaussian payoff distributions $\{\mathcal{N}(\mu_i,1)\}_{i\in[n]}$, let $\{w^t\}_{t>0}$ be any sequence of attention window realizations, potentially chosen adversarially and adaptively. Then given any confidence $0<\delta\leq 1$, for any item $i\in[n]$ and trial $t$, we define the good event 
\[\mathcal{E}_{i,t} := \left|\mu_i - \frac{r^t_i}{N^t_i}\right| \leq \sqrt{\frac{\log(4nt^2/\delta)}{N^t_i}},\]
where $r^t_i$ is the cumulative payoff received from item $i$ until trial $t$, and $N^t_i$ is the number of times item $i$ was selected until trial $t$. Then the event $\mathcal{E} := \cap_{i\in [n],t>0} \mathcal{E}_{i,t}$ occurs with probability at least $1-\delta$.
\end{lemma}
\begin{proof}
Consider any fixed item $i\in [n]$. For the purpose of analysis, we consider the following thought experiment to model the stochastic process that generates the payoff of item $i$ - the payoffs for this item are sampled independently from its underlying payoff distribution ahead of time and are written on an infinite tape. If this item is selected in some trial (which will be a function of the permutation $\pi^t$ played in that trial, as well as the attention window $w^t$ realized in that trial), the algorithm simply reads the payoff from the next cell on the tape. Now observe that for any fixed value $N$ of the number of times item $i$ is selected, the probability that the empirical average deviates significantly from the true mean $|\mu_i -\hmu_i| > \sqrt{\log(4nt^2/\delta)/N}$ is at most $2\delta/(4nt^2)$ by Hoeffding's inequality (Theorem~\ref{thm:hoeffding}). Furthermore, observe that this probability can be further upper bounded by $\delta/(2nN^2)$, as $N\leq t$ for all $t$. By taking a union bound over all $N>0$, and all possible items $i\in [n]$, we have that the probability that this event occurs for some item $i\in [n]$ for some value of $N$ is at most $\delta$. Moreover, this bound holds independent of the realization of the sequence of attention windows $\{w^t\}_{t>0}$. Our claimed bound follows by observing that this is exactly the probability of the event $\neg \mathcal{E}$.          
\end{proof}

From this point on, we shall condition on the good event $\mathcal{E}$, following which the regret incurred by our algorithm will be bounded deterministically. We shall focus on the case with stationary utilities, as it is a more fine grained analysis. The proof of the general case follows from a cruder analysis of the same.

\vspace{-5mm}
\begin{proof}[Proof of Theorem~\ref{thm:stochastic_main}]
We begin by decomposing the cumulative regret incurred by our algorithm into regret due to playing any item $j\in [n]$. For any fixed optimal permutation $\pi \in \Pi^*$, and any attention window of length $w$, let $s(w) \in S$ be the highest utility undominated item with rank at most $w$ in $\pi$, i.e. $s(w): \pi^{-1}(s(w))\leq w$. Observe that this item is uniquely defined by the family of permutations $\Pi^*$ and the attention window length $w$, and is independent of the specific choice of permutation $\pi \in \Pi^*$. This is precisely the item that would be played by any optimal permutation when an attention window of size $w$ realizes. Then we have that the cumulative regret of our algorithm over any time horizon of length $T$ is given by
$\sum_{t=1}^T R(\pi^t;w^t) = \sum_{j\in [n]}\sum_{t=1}^T \indic{y^t = j}(\mu_{s(w^t)} - \mu_{j})$

We shall now bound the total regret incurred due to playing some fixed item $j\in [n]$ when some undominated item $s_i\in S$ where $\mu_{s_i}>\mu_j$ could have been played, i.e. trials where $y^t = j$ and $s(w^t) = s_i$. We refer to such an event as an \emph{inversion} between $s_i$ and $j$. Observe that a necessary condition for such an inversion to occur, is that item $j$ is placed in an unblocked position (i.e. there is no higher-utility item placed before $j$ in $\pi^t$) in addition to either one of (1) $j$ being placed before $s_i$ in $\pi^t$, or (2) $j$ being placed after $s_i$ but there is at least one more item with utility smaller than $u_{s_i}$ that is placed after $j$ in $\pi^t$ (i.e. item $j$ could have been pushed out of the attention window $w^t$ by placing this blocked item before $j$ instead, forcing the selection of $s_i$ under this attention window realization). Observe that due to the nature of our algorithm~\ref{alg:find_permutation} which after fixing the position of any empirically undominated item in the partial permutation places all items that are blocked by that item immediately after it, an inversion due to case (2) can never occur. The key idea in our regret analysis is an argument that \emph{charges regret to inversions}, which we show will occur a bounded number of times as opposed to the number of selections which can be arbitrary. 

For any pair of items $s_i\in S,j\in [n]$, we claim that the number of inversions between $s_i,j$ over any time horizon of length $T$ is at most $M_{s_ij} := 4\log (4nT^2/\delta)/\Delta^2_{s_ij}$. More precisely, we claim that an inversion between $s_i,j$ cannot occur when $N^t_j\geq M_{s_ij}$, where $N^t_j$ is the number of times item $j$ has been selected until trial $t$, which is a tighter condition since $N^t_j$ always upper bounds the number of inversions between $j$ and \emph{any other item} $s_{i'}$ until trial $t$. This follows by observing that we receive a sample from the inferior item $j$ every time an inversion occurs with any other superior item $s_{i'}$. Let $t'$ be the trial when $N_j$ first exceeds $M_{s_ij}$, and let $N^{t'}_{s_i}$ be the number of times item $s_i$ has been selected until trial $t'$. We have the following two possibilities: either (1) $N^{t'}_{s_i} \geq M_{s_ij}$ due to which (by event $\mathcal{E}$) it must be the case that $\hmu^t_{s_i} - c^t_{s_i} > \hmu^t_j + c^t_{j}$ for all subsequent trials $t>t'$, or (2) $N^{t'}_{s_i} < M_{s_ij}$, due to which (by event $\mathcal{E}$) it must be the case that $\hmu^t_{s_i} + c^t_{s_i} > \hmu^t_j - c^t_{j}$. In either case, in all subsequent trials after $t'$, our algorithm will always place $s_i$ ahead of $j$ if $j$ is not already blocked, precluding the possibility of an inversion between $s_i$ and $j$; in case (1), $s_i$ will always empirically dominate $j$, due to which $j$ cannot be added into the partial permutation in an unblocked position until $s_i$ has been placed first. In case (2), our algorithm prioritizes lower selected items amongst items where no item beats another with confidence, due to which it will always place $s_i$ ahead of $j$ if $j$ is not already blocked in all subsequent rounds until eventually case (1) occurs. 

Equipped with this bound on the number of inversions, we are now ready to prove our claimed regret bound. For a cleaner analysis, for any item $j\in [n]$, we define $M_{s_0j}:=0$, and $\Delta_{s_0j}$ to be some arbitrarily large positive number approaching infinity. We shall analyze the regret for undominated and dominated items separately. For any undominated item $s_j\in S$, we have
\begin{align*}
    \sum_{t=1}^T \indic{y^t = s_j}(\mu_{s(w^t)} - \mu_{s_j})& = \sum_{t=1}^T\sum_{i=1}^{j-1} \indic{\left(y^t = s_j\right) \land \left(M_{s_{i-1}s_j} < N^t_{s_j} \leq M_{s_is_j}\right)}\left(\mu_{s(w^t)} - \mu_{s_j}\right)\\
    &\leq \sum_{i=1}^{j-1}\sum_{t=1}^T \indic{\left(y^t = s_j\right) \land \left(M_{s_{i-1}s_j} < N^t_{s_j} \leq M_{s_is_j}\right)}\left(\mu_{s_i} - \mu_{s_j}\right)\\
    &\leq \sum_{i=1}^{j-1} ( M_{s_is_j} - M_{s_{i-1}s_j}) \cdot \Delta_{s_is_j}
    = \sum_{i=1}^{j-1} 4\cdot \Delta_{s_is_j}\cdot \log(4nT^2/\delta)\cdot\left(\frac{1}{\Delta_{s_is_j}^2}-\frac{1}{\Delta_{s_{i-1}s_j}^2}\right)\\
     &= \sum_{i=1}^{j-1} 4\cdot \log(4nT^2/\delta)\cdot\left(1+\frac{\Delta_{s_is_j}}{\Delta_{s_{i-1}s_j}}\right)\cdot \left(\frac{1}{\Delta_{s_is_j}}-\frac{1}{\Delta_{s_{i-1}s_j}}\right)\\
    &\leq \sum_{i=1}^{j-1} 8\cdot \log(4nT^2/\delta)\cdot \left(\frac{1}{\Delta_{s_is_j}}-\frac{1}{\Delta_{s_{i-1}s_j}}\right)
    \leq \frac{8}{\Delta_{s_{j-1}s_j}}\cdot \log(4nT^2/\delta),
\end{align*}
where first inequality follows from our earlier condition on an inversion between items $s_i,j$, which gives us that $\mu_{s(w^t)} < \mu_{s_i}$ for all trials $t$ where $N^t_j>M_{s_ij}$, the penultimate inequality follows from the fact that $\Delta_{s_{i-1}s_j}>\Delta_{s_is_j}$ for any $i$, and the last inequality follows from cascading the summation. 

The regret analysis for a dominated item $j\in D$ uses one additional observation: item $j$ is permanently eliminated after $N_j=M_{s_{j*}j}$, where $s_{j^*} = \argmax_{i\in S:u_i>u_j} \mu_i$ is the highest mean undominated item that dominates item $j$. This is because after item $j$ is selected $M_{s_{j^*}j}$ times, it will never be placed in an unblocked position in any subsequent permutation; either $j$ is already in a blocked position, or $s_{j^*}$ must precede $j$ in the permutation, which will then block $j$. Therefore, applying an identical calculus as before, we have that for any dominated item $j\in D$,  
\begin{align*}
    \sum_{t=1}^T \indic{y^t = j}(\mu_{s(w^t)} - \mu_{j}) &= \sum_{t=1}^T\sum_{i=1}^{j^*} \indic{\left(y^t = j\right) \land \left(M_{s_{i-1}j} < N^t_{j} \leq M_{s_ij}\right)}(\mu_{s(w^t)} - \mu_{j})\\
    &\leq \sum_{i=1}^{j^*}\sum_{t=1}^T \indic{\left(y^t = j\right) \land \left(M_{s_{i-1}j} < N^t_{j} \leq M_{s_ij}\right)}\left(\mu_{s_i} - \mu_{j}\right)\\
    &\leq \frac{8}{\Delta_{s_{j^*}j}}\cdot \log(4nT^2/\delta).
\end{align*}
Combining the above regret bounds for undominated and dominated items gives us the regret upper bound claimed in Theorem~\ref{thm:stochastic_main}. 

\end{proof}

The proof of Theorem~\ref{thm:stochastic_main_general} follows essentially identically, except that in our charging argument, we cannot consider just pairs of items with one item in the pair being undominated (since the set of undominated items can change from round to round). Our bound $M_{ij}$ on the number of inversions between any pair of items $i,j$ with $\mu_i<\mu_j$ still holds, which gives us that for any item $j$,
\begin{align*}
    \sum_{t=1}^T \indic{y^t = j}(\mu_{s(w^t)} - \mu_{j}) &= \sum_{t=1}^T\sum_{i=1}^{j-1} \indic{\left(y^t = j\right) \land \left(M_{ij} < N^t_{j} \leq M_{i+1j}\right)}(\mu_{s(w^t)} - \mu_{j})\\
    &\leq \sum_{i=1}^{j-1}\sum_{t=1}^T \indic{\left(y^t = j\right) \land \left(M_{ij} < N^t_{j} \leq M_{i+1j}\right)}\left(\mu_{i} - \mu_{j}\right)\\
    &\leq \frac{8}{\Delta_{j-1j}}\cdot \log(4nT^2/\delta).
\end{align*}
Summing up this quantity over all items $j\in [n]$, $j\neq 1$ gives us the regret upper bound claimed in Theorem~\ref{thm:stochastic_main_general}.

\subsection{The Lower Bound}
In this section we will present an instance-dependent lower bound 
for our problem in the special case of stationary utilities, showing that the upper bound given in Theorem~\ref{thm:stochastic_main} is tight.

We define an algorithm to be \emph{consistent}
if, for any $p < 1$, there exists a constant $C_p$
such that for sufficiently large $T$, $\E[R_T(\A)] \leq C_p T^p$, uniformly over all instances of the problem.

\begin{theorem}\label{thm:lb}
Given items $[n]$  with utilities $\{u_i\}_{i\in [n]}$ and Gaussian payoff distributions $\{\mathcal{N}(\mu_i, 1)\}_{i \in [n]}$ such that $\mu_i$'s are distinct, there 
exists a sequence of attention window realizations 
$\{w^t\}_{t\in [T]}$ such that 
the expected regret 
of any consistent algorithm $\A$ over this instance is lower bounded as:
\begin{align}
\label{eq:lb}
 \lim_{T \rightarrow \infty}\frac{\E[R_T(\A;\{w^t\}_{t\in [T]})]}{\log T} = \Omega \left( \sum_{i = 1}^{|S|-1} \frac{1}{ \Delta_{s_is_{i+1}}} + \sum_{i = 1}^{|S|} \sum_{j \in D_{s_i}} \frac{1}{ \Delta_{s_ij}}\right)
    \,,
\end{align}
where $\Delta_{i j} = \mu_i - \mu_j$, 
and the other quantities such as $S, s_i, D_{s_i}$ are defined according to Lemma~\ref{lem:opt}.
\end{theorem}

Note that our lower bound matches the upper bound 
given in Theorem~\ref{thm:stochastic_main} in the case of fixed utilities.
This shows that our algorithm achieves instance-wise optimality
in the case of fixed utilities. 
We also have an immediate corollary of the above theorem.
\begin{corollary}\label{cor:lb}
In the setting of Theorem~\ref{thm:lb}, if the 
instance is such that $u_1 < u_2 < \cdots < u_n$ and $\mu_1 > \mu_2 > \cdots > \mu_n$  then
\begin{align*}
\label{eq:lb2}
 \lim_{T \rightarrow \infty}\frac{\E[R_T(\A;\{w^t\}_{t\in [T]})]}{\log T} = \Omega \left( \sum_{i = 1}^{n-1}  \frac{1}{ \Delta_{i i+1}} \right) 
    \,.
\end{align*}
\end{corollary}
Note that this lower bound matches the upper bound 
in Theorem~\ref{thm:stochastic_main_general}. Our construction uses ideas developed for classical bandits (see Chapter 16 in \cite{lattimore2020bandit}), and we defer a formal proof to Appendix~\ref{app:lb-proof} in the interest of space.

\section{Adversarial Payoffs and Stochastic Windows}
\label{sec:adversarial}

In this section we consider a setting 
where the payoffs can be arbitrary or adversarial.
Specifically, the environment  
generates the payoff vectors $\{\vr^t\}_{t \in [T]}$
ahead of time.
At time $t$, when the user selects an item $y^t \in [n]$,
the corresponding payoff $r^t_{y^t}$ is revealed to the algorithm. 
We also assume that the window realizations 
are stochastic. 
Specifically, 
at each time $t$, the user draws $w^t$ according to 
Multinomial$(q_1, \cdots, q_n)$
where $\sum_{w \in [n]} q_w = 1$
and $q_w \in [0,1]$ is 
probability of realization of a window of size $w$. For the ease of exposition, we first restrict our attention to the case with fixed utilities. Then we show how our result easily extends to the general case. Because of that, the platform's goal is to minimize total (pseudo-)regret with respect to the optimal in-hindsight permutation, defined as follows:
\[
\textsc{Regret}(T)\triangleq \max_{\pi \in \mathcal{S}^{n-1}} \sum_{t\in [T]} \E[ r^t(\pi) - r^t(\pi^t)]
\,.
\]

\subsection{Warm Up: Sub-optimal Regret by Inducing Exploration}
\label{sec:eps-greedy}

In this section we consider a \emph{laziness} constraint on the 
window probabilities $(q_1, \cdots, q_n)$ 
and give a  regret upper bound of $O(T^{2/3})$
under this constraint using the $\epsilon$-greedy algorithm.
This laziness constraint requires that  $q_1 \geq q_2 \cdots \geq q_n$ and corresponds to a setting where the 
users are biased towards smaller consideration set or window sizes.

We show that under this constraint it is possible 
to induce an uniform exploration over items
which is required by the $\epsilon$-greedy algorithm 
in order to achieve a $O(T^{2/3})$ regret.
Formally, we show that it is possible to 
select a distribution
over permutations such that when a `lazy' user
is shown a random permutation drawn from this distribution
the effective 
selection probability for any item is $\frac{1}{n}$.
Intuitively, this constraint allows uniform exploration as it is possible to `incentivize' a `lazy' user to select the worst utility item with sufficient probability by placing it in the first position.

We first present discuss how to select a distribution
over permutations that induced uniform exploration
using a combinatorial argument.
For ease of exposition, let the items be indexed in increasing order  of utilities, i.e.\ $u_1 < u_2 < \cdots < u_n$.
For each $i \in [n] $, we  define the permutation 
$\pi_i = (i, i-1, \cdots, 1, i+1, i+2, \cdots, n)$.
We consider distribution over these permutations 
where $\pi_i$ has probability mass $\alpha_i>0$ with $\sum_{i=1}^n \alpha_i =1$.
We define  
$$\alpha_i := \frac{\sum_{j=1}^{i-1} q_j - (i-1) q_i}{n \cdot (\sum_{j=1}^{i-1} q_j)\cdot (\sum_{j=1}^{i} q_j)},
 \text{ with }
\alpha_1 = \frac{1}{n q_1}.$$

It is very easy to observe that $\alpha_i > 0$ for all $i \in [n]$
using the laziness assumption.
We now establish that $\{\alpha_i\}_{i=1}^n$ induces uniform item selection probabilities using an induction argument.
Since, item $1$ is only selected in $\pi_1$ w.p.\ $q_1$
it is easy to observe that $p_1 = q_1\alpha_1 = 1/n$.
Now, let us assume that $p_{i-1} = 1/n$ holds for some $i -1$.
We will show that $p_{i} = 1/n$. 
In general an item $i$ is not selected in permutations
$\pi_{i+1},\cdots \pi_{n}$, and is selected w.p.\ $q_i$ 
in permutations $\pi_{1},\cdots \pi_{i-1}$
and w.p.\ $q_1$ in $\pi_1$.
Hence, we get 
\begin{align*}
p_i = \alpha_i \cdot \sum_{j=1}^i q_j + \sum_{j=1}^{i-1} \alpha_j \cdot q_i 
&= \alpha_i \cdot \sum_{j=1}^i q_j+  \alpha_{i-1} \cdot q_i + \sum_{j=1}^{i-2} \alpha_j \cdot q_i \cdot \frac{q_{i-1}}{q_{i-1}} \\
&= \alpha_i \cdot \sum_{j=1}^i q_j +  \alpha_{i-1} \cdot q_i + (\frac{1}{n} - \alpha_{i-1} \sum_{j=1}^{i-1} q_j) \cdot \frac{q_{i}}{q_{i-1}} = \frac{1}{n}
    \,,
\end{align*}
where the third equality holds by the induction hypothesis
and the last equality holds by plugging the values of $\alpha$
and some algebra. Finally, the fact that $\sum_{i}\alpha_i = 1$
can easily be established by the fact that $\sum_i p_i = 1$.

Given this distribution over permutations the algorithm is simple:
at each round $t$, toss a coin with probability of heads
being $O(T^{-1/3})$. If this coin turns up head
then sample $i \sim$ Multinomial$(\alpha_1, \cdots, \alpha_n)$
and play the permutation $\pi_i$. This 
induces a uniform distribution over item selections.
If the coin turns up tails, then play the optimal permutation with respect to 
the estimated payoffs from previous trials.
Using a standard argument for the $\epsilon$-greedy algorithm
one can easily show that the total regret 
of exploration as well as exploitation 
is $O(T^{2/3})$.

In this section we showed that the uniform distribution 
is admissible under the laziness assumption. 
In the next section we give a characterization
of the space of all admissible item distributions,
which will eventually lead to a $O(\sqrt{T})$ regret algorithm.

\subsection{Admissible Selection Probabilities }
\label{sec:polytope}

In this section, we introduce the notion of ``admissible'' item selection probabilities, which at a high level, are an alternate representation of the feasible set of actions that can be performed in any trial in our model. Consider any fixed permutation $\pi$ over items. Then observe that since the realized attention window is an i.i.d. random variable drawn from the distribution ${q}\in \Delta_n$ (where $\Delta_n$ is the $n$-dimensional probability simplex), the selected item is also an i.i.d. random variable drawn from the distribution $P(\pi){q}\in \Delta_n$, where $P(\pi) \in \{0,1\}^{n\times n}$ is the ``selection matrix'' of $\pi$, the entries $P_{i,w}(\pi)$ of which indicate whether item $i$ is selected in permutation $\pi$ when the realized attention window is of size $w$. Therefore, we can equivalently view a permutation as a categorical distribution over items (more precisely, the one induced by the distribution over attention windows on that permutation) and consequently, view the action of displaying a permutation as choosing a distribution over items instead. The objective then is to be competitive with the best fixed distribution over items in hindsight. However, the chosen distributions cannot be arbitrary --  they must correspond to a distribution induced by some permutation (or some convex combination of them since an algorithm can also play ``mixed strategies'' in any trial). We refer to such feasible categorical distributions as ``admissible'' selection probabilities. In this section, we show that this set $\mathcal{P}$, despite having factorially many vertices, can be succinctly characterized by only a few constraints.

In order to describe this set, we start by introducing some notation. In the remainder of this section, we shall find it convenient to label items (without loss of generality) in increasing order of utility, and use comparisons between items to refer to comparisons between their utilities i.e. $\forall i,j\in [n]$, $i<j \equiv u_i<u_j$. As mentioned earlier, any permutation $\pi$ can equivalently be represented as a ``selection'' matrix $P(\pi)\in \{0,1\}^{n\times n}$, where the rows correspond to items, and the columns correspond to the realized lengths of the windows. The entries of this matrix $P_{i,w}(\pi) = \indic{\pi^{-1}(i) \leq w \land \forall j\in [n]: \pi^{-1}(j) \leq w, i > j}$ indicate whether an item $i$ is selected when the realized attention window is of size $w$, which occurs iff item $i$ has the highest utility (highest label) amongst the first $w$ items in permutation $\pi$. Since the categorical distribution induced by $\pi$ can now be represented as $P(\pi){q}$, characterizing the set of all admissible selection probabilities reduces to characterizing the set of all valid selection matrices, i.e. selection matrices $P=P(\pi)$ that are realized by some permutation $\pi$, or a convex combination thereof. We claim that this set $\mathcal{P}$ is as follows: 
\begin{align}
 \label{eq:admissible_set}
\mathcal{P} = \{P\} \text{ such that }&\\
&~0\leq P_{i,w}\leq 1 &\forall{i,w\in [n]}  \tag{C.1}\label{con:c1}\\
&\sum_{i=1}^nP_{i,w} = 1 &\forall{w\in [n]} \tag{C.2}\label{con:c2}\\
& ~P_{i,w} = 0 &\forall{i,w\in [n]: i<w}  \tag{C.3}\label{con:c3}\\
&\sum_{i=n-k}^nP_{i,w} \leq \sum_{i=n-k}^n P_{i,w'} &\forall{k\in [n-1],~w,w'\in [n]: w<w'}  \tag{C.4}\label{con:c4}
\end{align}
These constraints intuitively describe the following properties of item selection probabilities (more generally induced by a convex combination of selection matrices):
\begin{description}
\item (\ref{con:c1}) The probability of selecting an item $i$ under a realized window length $w$ must be non-negative and upper bounded by $1$.\item (\ref{con:c2}) Exactly one item must be selected for every window length $w$.\item (\ref{con:c3}) Item $i$ (i.e. $i^{th}$ lowest utility) cannot be selected in any window of length strictly larger than $i$.\item (\ref{con:c4}) The selection probability of some item from the set of the $k$ highest utility items can only increase with increasing window lengths (larger window lengths are increasingly favourable to high utility items).
\end{description}

The easy direction, i.e. any selection matrix $P(\pi)$ induced by some permutation $\pi$ must belong to this set, follows directly from the definition of $P(\pi)$. Observe that $$P_{i,w}(\pi) = \indic{\pi^{-1}(i) \leq w \land \forall j\in [n]: \pi^{-1}(j) \leq w, i > j},$$ due to which it is a $\{0,1\}$ matrix with exactly one $1$ in every column, satisfying constraints (\ref{con:c1}) and (\ref{con:c2}), the entries $P_{i,w}(\pi) = 0$ for any $w>i$ since by the pigeonhole principle, there must exist some item $j>i$ with $\pi^{-1}(j)\leq w$ satisfying constraint (\ref{con:c3}), and due to the laminar nature of the windows, i.e. for any $w<w'$, $\{i\in [n]:\pi^{-1}(i) \leq w\} \subset \{i\in [n]:\pi^{-1}(i) \leq w'\}$, the indicator can only ``shift downwards'' across columns, i.e. the labels of the items (utilities) selected under increasing window lengths can only increase, satisfying constraint (\ref{con:c4}).

To prove the hard direction, i.e. any feasible matrix $P\in \mathcal{P}$ (potentially fractional) can be decomposed into a convex combination of selection matrices of permutations, we will in fact provide a stronger constructive guarantee: there is an efficient algorithm that, given as input any feasible matrix $P$, returns at most $n^2$ many permutations along with their corresponding weights such that the weighted average of their selection matrices produces $P$.

\subsubsection{The Rounding Algorithm}

Our rounding algorithm for decomposing a fractional selection matrix $P$ into a convex combination of integral selection matrices of permutations is a recursive peeling algorithm. In each round, our algorithm identifies an integral matrix supported on the non-zero coordinates of this fractional matrix and peels off its maximal contribution from the fractional matrix. This integral matrix is identified as follows: for each window length $w\in [n]$, we identify the lowest utility item $i_w$ that has a non-zero probability $P_{i_w,w}>0$ of being selected in this window. These items $i_w$ then define an integral matrix $\hat{P}$, which has a $1$ in coordinates $(i_w,w)$ for each window length (column) $w\in [n]$ and $0$ everywhere else. As we will see later in the proof of correctness of this algorithm, this construction always produces a valid integral selection matrix, i.e. there is some permutation $\pi$ such that this integral matrix $\hat{P} = P(\pi)$ corresponds to the selection matrix of $\pi$. Moreover, this permutation $\pi$ can be efficiently recovered from $\hat{P}$. The weight $m$ of this integral (selection) matrix $\hat{P}$ is then its maximal contribution $m = \min_{w\in [n]}P_{i_w,w}$ in $P$. This weighted integral matrix is then removed from $P$, following which the residual is rescaled to project it back into the feasible set $\mathcal{P}$, i.e. $P$ is updated to be $(1-m)^{-1}(P-m\cdot \hat{P})$. We then recurse on this rescaled residual matrix, stopping when the matrix becomes all $0$. The depth of this recursion is bounded by $n^2$, since in each round, by maximally removing an integral selection matrix, we reduce the number of non-zero entries in $P$ by at least $1$. As a consequence, the support of the resultant decomposition of $P$ into integral selection matrices/permutations is also bounded by $n^2$. The formal rounding algorithm is described in Algorithm~\ref{alg:rounding}, with the algorithm for reconstructing a permutation from an integral selection matrix being described in Algorithm~\ref{alg:round-integral}.          

\begin{algorithm}[t]
\caption{\textsc{Round Fractional Selection Matrix (RFSM)}}
\label{alg:rounding}
\begin{algorithmic}[1]
\State \textbf{Input:} Feasible selection probability matrix $P\in \mathcal{P}$.
\State Decomposition into (weight, permutation) pairs $R\leftarrow \emptyset$, residual probability mass $p \leftarrow 1$.
\While {$P \neq \boldsymbol{0}$}
\State For every $w\in [n]$, define $i_w \leftarrow \min \{i\in [n]:P_{i,w}>0\}$. Then $m \leftarrow \min_{w\in [n]} P_{i_w,w}$.
\State Set integral selection matrix $\hat{P}\in \{0,1\}^{n\times n}$, where $\hat{P}_{i,w} = 1$ if $i=i_w$ and $0$ otherwise.
\State Set $\pi \leftarrow \IP (\hat{P})$. 
\State Update decomposition $R\leftarrow R \cup ( p\cdot m,\pi)$.
\State Update residual mass $p\leftarrow p\cdot (1-m)$
\State Update and scale residual selection probabilities $P \leftarrow (1-m)^{-1}(P-m\cdot \hat{P})$.
\EndWhile
\State \textbf{Return:} $R$
\end{algorithmic}
\end{algorithm}
\begin{algorithm}
    \begin{algorithmic}[1]
    \caption{\textsc{Integral Permutation}}
    \label{alg:round-integral}
    \State \textbf{Input:} Feasible integral selection probability matrix $P\in \{0,1\}^{n\times n} \land P\in \mathcal{P}$.
    \State Remaining items $A\leftarrow [n]$, partial permutation $\pi \leftarrow \emptyset$.
    \For{$w = 1,\ldots,n$}
        \State $i^* \leftarrow i:P_{i,w} = 1$ \Comment{$i^*$ is the item selected in a window of length $w$}
        \If{$i^* \notin A$} \Comment{$i^*$ already appears earlier in $\pi$}
        \State $i^* \leftarrow \min \{i\in A\}$ 
        \EndIf
        \State Update $\pi \leftarrow \pi + i^*$, $A\leftarrow A\setminus i^*$. \Comment{Assign item $i^*$ to the $w^{th}$ position in $\pi$}
    \EndFor
    \State{\textbf{Return:} $\pi$}
\end{algorithmic}
\end{algorithm}

The following theorem describes the main result of this section.
\begin{theorem}
\label{thm:decomposition}
Any matrix $P\in \mathcal{P}$ can be decomposed into a convex combination $P=\sum_{r=1}^z p_r \cdot P(\pi_r)$ of selection matrices $\{P(\pi_r)\}_{r\in [z]}$ of at most $z-n+1$ permutations $\{\pi_r\}_{r\in [z]}$, where $z$ is the number of non-zero entries in $P$. Moreover, there is an efficient algorithm for finding this decomposition $\{p_r,\pi_r\}_{r\in [z]}$.   
\end{theorem}

\begin{proof}
We shall prove this claim via induction on $z$, the number of non-zero entries in matrix $P\in \mathcal{P}$. To show the base case where $z=n$ (observe that $z\geq n$ as every column sums to $1$), we will show that any integral matrix $P\in \mathcal{P}$ corresponds to a selection matrix $P(\pi)$ of some permutation $\pi$, which can be reconstructed efficiently given $P$ (Algorithm~\ref{alg:round-integral}). In order to do so, we shall find it useful to first understand the structure of any integral matrix $P$. For every item $i\in [n]$, we define the set $W_i := \{w\in [n]:P_{i,w} = 1\}$, which corresponds to the window lengths under which this item would be selected. Note that this set can be empty for some items (the items that never get selected). If on the other hand, we have that $|W_i|>1$ for some item $i\in [n]$, it must be the case that $W_i$ must contain consecutive values due to the fact that $P$ is integral and satisfies constraint (\ref{con:c4}). Furthermore, it must be the case that $\forall w\in W_i, w\leq i$ due to constraint (\ref{con:c3}), and $\{W_i\}_{i\in [n]}$ must partition $[n]$ due to constraint (\ref{con:c2}). Given these sets, we now iteratively construct a permutation $\pi$, starting with an empty permutation as follows: for each position $w$ in increasing order of positions, we identify the item $i$ whose set $W_i$ contains $w$. Observe that such an item $i$ must exist since the sets $\{W_i\}_{i\in [n]}$ partition $[n]$. If $i$ has not already been placed in the partial permutation $\pi$ thus far, we assign $i$ to position $w$ in our permutation. Otherwise, we place in position $w$, the lowest utility item $i'$ that has not yet been placed in the permutation so far. We claim that this item $i'$ must be such that $i'<i$ and $W_{i'}=\emptyset$. To see this, consider the top left sub-matrix $P_{1,1} \leq P_{a,b}\leq P_{i,w}$ which must have at least as many rows as columns since $i\geq w$. Since exactly $w-1$ items have been placed in $\pi$ so far, which include all items $i''$ with an entry $1$ in some in some column $<w$ (we assign positions to items in increasing order of positions), and there are exactly $w$ entries with value $1$ in this submatrix, the pigeonhole principle guarantees the existence of at least one item $i'$ (which by definition satisfies $i'<i$) with an all $0$ row in this submatrix (which gives us $W_{i'} = \emptyset$), that has not been placed in $\pi$ so far. Given this construction, it is now straightforward to verify that this will produce a permutation $\pi$ with $P(\pi) = P$, proving our base case.

For the induction step, we shall prove our claim for any matrix $P\in \mathcal{P}$ with $z$ non-zero entries, assuming our claim holds for all $n\leq z'<z$. The proof of this effectively reduces to showing two things: (a) the integral matrix $\hat{P}$ extracted from $P$ is contained in $\mathcal{P}$, and (b) the residual matrix $R=(1-m)^{-1}(P-m\hat{P})$, if it is not an all $0$ matrix, is also contained in $\mathcal{P}$, where quantities $\hat{P}$ and $m$ are defined in Algorithm~\ref{alg:rounding}, and discussed in its description. The former allows us to leverage our base case -- since the integral matrix $\hat{P}$ has exactly $n$ non-zero entries, it corresponds to the selection matrix $P(\pi_0)$ of some permutation $\pi_0$. The latter allows us to leverage our induction hypothesis -- since the residual matrix $R$ must contain at least one fewer non-zero coordinate than $P$ (in particular, the non-zero entry $P_{i_w,w}$ for which the minimum value $m$ is achieved becomes a zero entry in $R$), it can be decomposed into a convex combination $R=\sum_{r=1}^{z'}p_r\cdot P(\pi_r)$ of selection matrices $\{P(\pi_r)\}_{r\in [z']}$ of $z' \leq (z-1)-n+1$ permutations $\{\pi_r\}_{r\in [z']}$. We can combine these two properties to decompose $P = m\cdot P(\pi_0) + \sum_{r=1}^{z'} (1-m)\cdot p_r\cdot P(\pi_r)$ into a convex combination of selection matrices of $ z' + 1 \leq z-n+1$ permutations, completing the induction step.   

To show that $\hat{P}\in \mathcal{P}$, observe that it is a $\{0,1\}$ matrix with exactly one $1$ in each column, satisfying constraints (\ref{con:c1}) and (\ref{con:c2}), and since $P\in \mathcal{P}$, it must be that $i_w\geq w$ for each $w\in [n]$ satisfying constraint (\ref{con:c3}). To show constraint (\ref{con:c4}) is satisfied, we claim that for any $w<w'$, it must be that $i_w\leq i_{w'}$. To see this, assume for the sake of contradiction that this does not hold for some pair $w,w'$, i.e. $i_{w'}<i_{w}$ where $w<w'$. Then observe that we have $\sum_{i=i_{w}}^n P_{i,w'} \leq 1 - P_{i_{w'},w'} < 1$, whereas $\sum_{i=i_w}^n P_{i,w} = 1$, which gives us that $\sum_{i=i_{w}}^n P_{i,w'} < \sum_{i=i_{w}}^n P_{i,w}$ for $w<w'$, violating constraint (\ref{con:c4}) for matrix $P$, contradicting the assumption $P\in \mathcal{P}$. 

We will now show that the rescaled residual matrix $R = (1-m)^{-1}(P-m\hat{P}) \in \mathcal{P}$. We have that for any coordinate $[P-m\hat{P}]_{i,w}$, it is equal to $P_{i,w}-m$ if $i=i_w$ and $P_{i,w}$ otherwise. By definition of $m = \min_{w\in [n]} P_{i_w,w}$, we have that $[P-m\hat{P}]_{i,w} \geq 0$. It must also be the case that $[P-m\hat{P}]_{i,w} \leq 1-m$, since for every $i\neq i_w$, $P_{i,w}\leq 1-P_{i_w,w}$ which follows by definition of $i_w$ and the fact that $P$ satisfies (\ref{con:c2}). Moreover, this also gives us that for any $w$, $\sum_{i=1}^n [P-m\hat{P}]_{i,w} = \sum_{i=1}^n P_{i,w} - m = 1-m$. Therefore, $R$ satisfies constraints (\ref{con:c1}) and (\ref{con:c2}). $R$ also satisfies constraint (\ref{con:c3}) quite straightforwardly as $P$ satisfies constraint (\ref{con:c3}) and any $0$ coordinate in $P$ continues to remain a $0$ coordinate in $R$. To show that $R$ also satisfies constraint (\ref{con:c4}), consider any pair of columns $w<w'$. The only case we need to consider is $i_w < n-k \leq i_{w'}$, since in all other cases, constraint (\ref{con:c4}) is satisfied straightforwardly; observe that for matrix $P-m\hat{P}$, in the case that $n-k \leq i_w$, we have $\sum_{i=n-k}^n [P-m\hat{P}]_{i,w} = \sum_{i=n-k}^n [P-m\hat{P}]_{i,w'} = 1-m$, and in the case that $n-k>i_{w'}$, we have $\sum_{i=n-k}^n [P-m\hat{P}]_{i,w} = \sum_{i=n-k}^n P_{i,w}$ and $\sum_{i=n-k}^n [P-m\hat{P}]_{i,w'} = \sum_{i=n-k}^n P_{i,w'}$, and since $P$ satisfies constraint (\ref{con:c4}), it must be that $\sum_{i=n-k}^n [P-m\hat{P}]_{i,w} \leq \sum_{i=n-k}^n [P-m\hat{P}]_{i,w'}$. Now if $k$ is such that $i_w<n-k\leq i_{w'}$, we have  $\sum_{i=n-k}^n [P-m\hat{P}]_{i,w} = \sum_{i=n-k}^n P_{i,w} \leq 1-P_{i_w,w} \leq 1-m$. Furthermore, we have $\sum_{i=n-k}^n [P-m\hat{P}]_{i,w'} = \sum_{i=n-k}^n P_{i,w'} - m = 1 -m$. Therefore, we have $\sum_{i=n-k}^n [P-m\hat{P}]_{i,w} \leq 1-m = \sum_{i=n-k}^n [P-m\hat{P}]_{i,w'}$. Therefore, $R$ also satisfies constraint (\ref{con:c4}), which proves $R\in \mathcal{P}$, completing the proof of our claim.
\end{proof}

The consequence of \Cref{thm:decomposition} is that we can efficiently simulate any ``admissible'' selection probabilities by randomly sampling a permutation $\pi \sim \{p_r,\pi_r\}_{r\in [n]}$ from the above decomposition.

\subsection{The Algorithm}
\label{sec:reduction-online-linear-optimization}
We are now ready to present our algorithm that achieves $O(\sqrt{T})$
regret. Our algorithm is based on online linear optimization
with bandit feedback over the space of admissible probabilities.
We use the characterization of this space of admissible probabilities 
from the previous section in order to decompose an induced probability 
distribution over items into a distribution over permutations. We first define the problem of linear optimization with 
bandit feedback.
\begin{definition}[Bandit Linear Optimization (BLO)]
Let the time-horizon be $T$ and action set of the player be $\C \subseteq \R^d$.
In each round $t \in [T]$,
the player chooses an action $\va^t \in \C$ and the 
adversary simultaneously chooses a loss function $\vl^t \in \R^d$. The player then observes the loss $(\vl^t)^\top \va^t$.
The performance of the player is evaluated in terms of
regret $R_T$, defined as $R_T = \E [\sum_{t=1}^T (\vl^t)^\top \va^t] - \min_{\va\in \C} \E[\sum_{t=1}^T (\vl^t)^\top \va^*]$.
\end{definition}

This problem has been well-studied in online learning literature  
\citep{bubeck2012regret, hazan2016introduction},
and there are several $\BLO$ algorithms 
that achieve a regret of $\Ot(\sqrt{T})$.
In our result we use the well-known Online Stochastic Mirror Descent
($\osmd$) algorithm.
We assume use $\osmd$ in a black-box manner assuming access 
to two functions $\bloact(t)$ that outputs the action $\va^t \in \C$
at time $t$, and $\blofeed(t, \hat{\vl}^t)$ that takes as input an estimate of $\hat{\vl}^t$ of the 
loss at time $t$.

We first describe the reduction for the case of 
fixed utilities across time. 
Given the set of utilities $\{u_i\}_{i=1}^n$,
the action space $\C \subseteq \Delta_{n}$ is the space of all admissible item selection probabilities, i.e.\
$\C = \{p \in \Delta_n: \exists  P \in \cP \text{ with } p = P q\}$
where $\cP$ is defined in Equation~\ref{eq:admissible_set}.
It is easy to see that this set is closed and convex 
as it is the convex hull of item selection probabilities
induced by the permutations $\pi \in \cS_{n-1}$.
Also, the loss vector $\ell^t$ at time $t$ is the 
defined as $\ell^t = -r^t$.
At each time $t \in [T]$ we call $\bloact(t)$
which outputs a vector $p^t \in \C$. 
Now, since we cannot directly select an item according to 
$p^t$,
we need to find a distribution over permutations 
such that playing a random permutation according to this distribution
induces item selection probabilities $p^t $.
In order to find such a distribution 
we will utilize the decomposition algorithm 
from the previous section.
This algorithm requires an item selection 
matrix $P \in \cP$ as input. 
However, one can find a matrix $P \in \cP $ such that 
$p^t = Pq$ in polynomial time 
by describing $P$ using a set of linear constraints. 
We then use the $\decompose$ algorithm (Algorithm~\ref{alg:rounding}) over $P$ to find a distribution $\D$ over permutations. 
Next, we play $\pi^t \sim \D$ and observe the payoff $r^t_{y^t}$ of the selected item.
Note that  $y^t$ is distributed according to $p^t$.

We now need to provide unbiased estimates of the loss $\ell^t$
to $\osmd$. 
Let $\hat{\ell}^t_i = -r^t_{i}/p^t_i$ for $i = y_t$
and $\hat{\ell}^t_i = 0$ for $i \neq y_t$.
We then have that
$\E[\hat{\ell}^t_i] =  p^t_i \cdot \frac{-r^t_y}{p^t_i} + (1-p^t_i) \cdot 0 = \ell^t_i$

We present a pseudo-code of our algorithm in Algorithm~\ref{alg:blo}.

\begin{algorithm}
    \begin{algorithmic}[1]
    \caption{\textsc{Ranking for Limited-Attention Users using BLO}}
    \label{alg:blo}
    \State \textbf{Input:} items $[n]$, space of admissible selection probabilities $\C \subseteq \Delta_n$
    \State Initialize the $\osmd$ algorithm over the action space $\C$ 
    \For{$t = 1,\ldots,T$}
        \State $\vp^t \leftarrow \bloact(t)$
        \State Find $P\in \cP$ such that $p^t = Pq $
        \State $\D \leftarrow \decompose(P)$  (Algorithm~\ref{alg:rounding})
        \State Sample $\pi^t \sim \D$ and play $\pi^t$
        \State Observe the payoff $r_{y^t}^t$ of the selected item $y^t$
        \State $\blofeed(t, \hat{\ell}^{t})$
    \EndFor
\end{algorithmic}
\end{algorithm}

We now state Theorem 5.7 from \cite{bubeck2012regret}
that bounds the regret of OSMD.

\begin{theorem}[OSMD]
For any closed convex action set $\C \subseteq \R^d$, the OSMD algorithm 
with a regularizer $F_{\psi} : \C \rightarrow \R_+ \cup \{+\infty\}$
over a sequence of loss vectors $\{\ell^t\}_{t=1}^T$ 
with access to loss estimates $\hat{\ell}^t$ such 
that $\E[\hat{\ell^t}] = \ell^t$ satisfies
$ R_T \leq 2 \sqrt{2 Tn}$,
 for $\psi(x) = x^{-2}$ and $F_{\psi}(a) = \sum_{i=1}^n \int_{0}^{a_i} \psi^{-1}(s) d s$.
\end{theorem}

Using the above theorem we directly get a regret upper bound
for Algorithm~\ref{alg:blo}:
\begin{align*}
\max_{\pi \in \mathcal{S}^{n-1}} \E[ \sum_{t =1}^T (r^t)^\top p_{\pi}] - \E[ \sum_{t=1}^T (r^t)^\top p_{\pi^t}]
&= 
\max_{\pi \in \mathcal{S}^{n-1}} \E[ \sum_{t =1}^T (r^t)^\top p_{\pi}] - \E[ \sum_{t=1}^T (r^t)^\top p^t] \\
&= 
\max_{p \in \conv(\{p_{\pi}\}_{\pi \in \mathcal{S}^{n-1}})} \E[ \sum_{t =1}^T (r^t)^\top p] - \E[ \sum_{t=1}^T (r^t)^\top p^t] \\
& =\E [\sum_{t=1}^T (\vl^t)^\top p^t] - \min_{p\in \C} \E[\sum_{t=1}^T (\vl^t)^\top p]
\leq 2 \sqrt{2Tn}
    \,,
\end{align*}
where the first equality holds as
$\pi^t$ is selected in a way that the induced item distribution is 
$p^t$,
the second equality follows due to the well-known fact that any linear function is maximized at one of the vertices
of a convex constraint set,
the third equality follows from the definition of $\ell^t$.

Note that the regret achieved by this algorithm is tight 
because an $\Omega(\sqrt{Tn})$ lower bound is known for the 
multi-armed bandit problem which is a special case of 
our problem when $q_1 = 1$.

Now, consider the case of changing utilities across time. 
Recall, that the observed set of the utilities at time $t$ 
is $\{u^t\}_{i=1}^n$.
Without loss of generality, we will assume that 
$u^t_i \in \{1, 2, \cdots, n\}$ for all $t \in [T]$ and $i \in [n]$.
At time $t$, let $\sigma^t: [n] \rightarrow [n]$ be a permutation that
maps items to its utilities, i.e.\ 
$\sigma^t(i)=j$ if $u^t_i = j$.
We now run the same algorithm as before except 
one change in the computation of loss estimates:
if $\hat{\ell}_i^t = -r^t_i/p^t_i$ if $i = \sigma^t(y_t)$
and $\hat{\ell}_i^t = 0$ otherwise. 
The same regret bound as before holds 
against the best fixed mapping from utilities to permutations
by the simple observation that losses are now
associated with utility values rather than 
items.

\section{Conclusion}
\label{sec:conclusion}
In this paper, we studied a relatively general model for online learning-to-rank for limited-attention users whose preferences may not align with the goals of the platform. We considered two settings for online learning with bandit feedback: stochastic item payoffs with adversarial attention window lengths, and adversarial item payoffs with stochastic attention window lengths (with known distribution). In both settings, we designed interpretable algorithms that achieve low regret against a natural benchmark. In the stochastic payoff setting, our algorithm achieved an instance-dependent $O(\log T)$ regret that we further showed was optimal via a matching regret lower bound. In the adversarial payoff setting, our algorithm achieved a worst-case $O(\sqrt{T})$ regret, which was also optimal due to known regret lower bounds for classical bandits.

{
\bibliographystyle{plainnat}
\bibliography{refs}
}

\appendix

\section{Technical Facts and Lemmas}

We first state this standard fact about the KL divergence of two Gaussian random variables. 
\begin{fact}
\label{fact:kl_gaussian}
The KL divergence  between two Gaussians 
$\KL(\N(\mu,1), \N(\mu',1)) = (\mu - \mu')^2$.
\end{fact}

We now state a classical inequality that is used in our lower bound construction. 
\begin{theorem}[Bretagnolle-Huber inequality]
\label{thm:bhi}
Let $P$ and $Q$ be probability measures on the same measurable space $(\Omega, \F)$, and let $A \in \F$
be an arbitrary event. Then,
\[
P(A) + Q(A^c) \geq \frac{1}{2} \exp(-\KL(P || Q))
  \,,
\]
where $A^c = \Omega \setminus A$ is the complement of event $A$.
\end{theorem}

\begin{theorem}[Hoeffding's inequality]
\label{thm:hoeffding}
Let $X_1,\ldots,X_n$ be $n$ i.i.d. $\sigma$-sub-Gaussian random variables with mean $\mu$. Then for any $t>0$, we have that
\[\Pr\left(\bigg|\sum_{i=1}^n X_i - \mu \bigg| > t  \right)\leq 2e^{-t^2/(n\sigma^2)}\]
\end{theorem}

\section{Proof of Theorem~\ref{thm:lb}}
\label{app:lb-proof}

The proof of our lower bound uses ideas from 
lower bound construction for the classical multi-armed bandit problem (see Chapter 16 in \cite{lattimore2020bandit}).

We will utilize the 
following divergence lemma.
in order to prove our bound.
\begin{lemma}
\label{lemm:divergence}
Given time-horizon $T$,
consider two instances of the problem that share item utilities $\{u_i\}_{i\in [n]}$ but have different 
payoff distributions $\nu = (P_1, \cdots P_n)$ and 
$\nu' = (P'_1, \cdots, P'_n)$.
Fix an algorithm $\A$ that knows the sequence of attention window realizations ahead of time. Let $\bP_\nu = \bP_{\nu\A}$
and $\bP_{\nu'} = \bP_{\nu' \A}$
be probability measures on realized sample paths 
of $\A$ under $\nu$,
and $\A$ under $\nu'$, respectively.
\[
\sum_{i\in [n]} \E_{\nu}[N_{i}(T)] \cdot \KL(P_i ~||~ P'_i) = \KL(\bP_\nu ~||~ \bP_{\nu'}) \,,
\]
where $N_i(T)$ denotes the number of samples of item $i$ in $T$ trials and 
$\KL$ is the Kullback-Leibler divergence between two probability measures.
\end{lemma}

The proof of this lemma follows directly from Lemma 15.1 in \cite{lattimore2020bandit}. We are now ready to prove our lower bound.

\begin{proof}[Proof of Theorem~\ref{thm:lb}]
Given a sufficiently large time-horizon $T$ that is divisible by $n$,
let the attention window realizations be defined as 
\[
w^t := i \text{ for }  i \in [n] \text{ and } t \in \{(i-1)\cdot \frac{T}{n}+1, (i-1)\cdot \frac{T}{n}+2, \ldots, i \cdot \frac{T}{n}\}
    \,.
\]
Note that these attention window realizations are
deterministic and are known by the algorithm ahead of time.
This can only improve the regret of the algorithm as it can simply choose to ignore this information.
Also, note that these attention window realizations are the same across all instances and do not provide any information 
about the structure of the optimal permutation. 
Hence, we will drop the dependence on $w^t$'s in the remainder 
of the proof.

Let $\I := ( \{u_k\}_{k\in [n]}, \{\mathcal{N}(\mu_k,1)\}_{k \in [n]})$ denote the given instance of the problem, and let $\nu = (\N(\mu_1,1), \cdots, \N(\mu_n,1))$ be the collection of payoff distributions. 
Recall that the family of optimal permutations under $\I$
denoted by $\Pi^*(\I)$ is 
$(s_1,D_{s_1},s_2,D_{s_2},\ldots,s_k,D_{s_k})$.
The proof of the lower bound contains $|S|$ phases
each corresponding one of the items in $S$.
The $i$-th phase considers the time interval $\{1, \ldots, T_i\} \subseteq [T]$ 
where $T_i$ is defined such that,
if one plays some permutation from $\Pi^*(\I)$ then 
all the selected items belong to $\{s_1, \cdots, s_i\}$ upto time $T_i$.
Specifically, since an item from $\{s_1, \cdots, s_i\}$
if and only if $w \leq \sum_{i' \in  [i]} (|D_{s_{i'}}| + 1)$, we have that $T_i = \argmax \{t\in [T]: w^t \leq \sum_{i' \in  [i]} (|D_{s_{i'}}| + 1)\}$.


We now consider phase $i$ for some $i \in [|S|]$. Fix an item  $j \in D_{s_i}\cup \{s_{i+1}\}$ if $i \neq |S|$, otherwise fix an 
item $j \in D_{s_i}$. 
Let $T' = T_i$.
We will modify the instance $\I$
to create another instance $\I'$
which has the same set of utilities but different 
means $(\mu'_1, \cdots, \mu'_n)$.
Let $\mu_k' = \mu_k$ for all $k \neq j$,  
and $\mu'_j = \mu_{s_i}+ \Delta_{s_i j} +\epsilon$ for an arbitrary $ \epsilon >0$ chosen such that $\mu'_j \neq \mu'_k$ for any $k \neq j$.
Let $\nu' = (\N(\mu'_1,1), \cdots, \N(\mu'_n,1))$.
Firstly, using Fact~\ref{fact:kl_gaussian}, we 
have that $\KL(\N(\mu_j, 1) ||  \N(\mu'_j,1) ) = (\Delta_{s_i j}+\epsilon)^2$.
Moreover, it is easy to observe that $\KL(\N(\mu_k, 1) ||  \N(\mu'_k,1) ) = 0$
for $k \neq j$.
Then, using Lemma~\ref{lemm:divergence} on $\nu$ and $\nu'$, we have that 
$\KL(\bP_{\nu } || \bP_{\nu' }) \leq \E_{\nu }[N_j(T')] \cdot (\Delta_{s_ij}+\epsilon)^2$.

Using Theorem~\ref{thm:bhi}, for any 
event $A$,
we have that 
\begin{align}
\label{eq:bhi}
\bP_{\nu}(A) + \bP_{\nu}(A^c) \geq \frac{1}{2} \exp \paren{- \KL (\bP_{\nu } || \bP_{\nu' })}
\geq \frac{1}{2} \exp \paren{- \E_{\nu }[N_j(T')] \cdot (\Delta_{s_ij}+\epsilon)^2}
\,.
\end{align}
We will now choose an appropriate event $A$ based 
on the number of times $j$ is selected in each instance.

Firstly, since $j \notin \{s_1, \cdots, s_i\}$,
it will not be selected by any permutation in $\Pi^*(\I)$ upto time 
$T'$ by the definition of $T'$.
Hence, if any other permutation $\pi \notin \Pi^*(\I)$
selects $j$ upto time $T'$ then it will incur 
regret at least $\Delta_{s_i j}$ per selection.
This is due to the fact that the expected payoff of $\Pi^*(\I)$ is at least $\mu_{s_i}$ for each time in $[T']$
while the expected payoff for $\pi$ will be $\mu_j$
upon selecting $j$, and $\mu_j < \mu_{s_i}$.

Observe that $j \notin D_{s_{i'}}$  
for any $i' < i$ as otherwise its position 
in $\Pi^*(\I)$ would have been ahead of $s_i$.
Now, consider the set of optimal permutations $\Pi^*(\I')$
under $\I'$. 
Since, $\mu'_{j} > \mu'_{s_i}$, $j$ acquires 
a higher position in $\Pi^*(\I')$ than $s_i$, 
i.e.\ the position of $j$ denoted by $p_j$ will be at most  
$\sum_{i' \in  [i-1]} (|D_{s_{i'}}| + 1)$.
This implies  that the attention window realizes to $p_j$
exactly $T/n$ times upto time 
$T'$ as 
$p_j \leq \sum_{i' \in  [i]} (|D_{s_{i'}}| + 1)$.
Now, using the fact that $j$ is not dominated by any other item of mean higher than $\mu'_j$ we 
get that $\Pi^*(\I')$ will select $j$ at least 
$T/n$ times.
Hence, if any permutation $\pi$ selects
$j$ lesser than $T/n$ times when an attention window 
of size $p_j$ realizes, then it will incur 
a positive regret for selecting any other item with mean smaller than $\mu'_j$.
Also, note that $\pi$
cannot select any item with mean higher than $\mu'_j$
when an attention window of size $p_j$ realizes because $\Pi^*(\I')$ could have done the same thing and strictly increase its payoff. 
Hence,  if $j$ is selected lesser than $T/n$ times by $\pi$ then it incurs regret at least 
$\Delta = \min\{\mu'_j - \mu'_{s_{i'}}: i' \in S \text{ such that } \mu'_{s_{i'}} < \mu'_j \}$ for 
each such time.
Note that $\Delta > 0$ by the definition of $\mu'_j$.

We now define event $A = \{N_j(T') > T/2n \}$. Let $R = \E_{\nu}[R_{T'}(\A)]$ and $R'= \E_{\nu'}[R_{T'}(\A)]$. 
We have
\begin{align*}
R+ R' &\geq  \frac{T}{2n}  \cdot (\bP_{\nu }(A) \Delta_{s_i j} + \bP_{\nu' }(A^c) \Delta) \\
&\geq  \frac{T}{2n} \cdot \min\{\Delta, \Delta_{s_i j}\}  \cdot (\bP_{\nu }(A)  + \bP_{\nu}(A^c) ) \\
&\geq  \frac{T}{4n} \cdot \min\{\Delta, \Delta_{s_i j}\}  \cdot \exp\paren{- \E_{\nu}[N_j(T')] \cdot (\Delta_{s_ij}+\epsilon)^2} 
    \,,
\end{align*}
where the first inequality above follows from the fact 
selecting $j$ for more than $T/2n$ times in $\I$
incurs regret at least $\Delta_{s_i j}$
and selecting $j$ for less than $T/2n$ times incurs
regret $\Delta$,
and the last inequality follows from Equation~\ref{eq:bhi}.

Rearranging and taking limit gives us
\begin{align*}
\lim_{T \rightarrow \infty} \frac{\E_{\nu }[N_j(T')]}{\log T} 
    &\geq \frac{1}{(\Delta_{s_ij} +\epsilon)^2} \lim_{T\rightarrow \infty} \frac{\log\paren{\frac{T \min\{\Delta_{s_ij}, \Delta\}}{4 (R+R')n }}}{ \log T} \\
    & = \frac{1}{(\Delta_{s_ij} +\epsilon)^2} \cdot \paren{
   1- \lim_{T \rightarrow \infty}  \frac{\log (R +R')}{ \log T}} 
    = \Omega\paren{\frac{1}{\Delta_{s_ij} ^2}}
        \,,
\end{align*}
where the first equality follows due to the fact that 
$\min\{\Delta_{s_i j}, \Delta\}$ and $n$
do not depend on $T$,
and the last equality follows by the consistency of 
$\A$ and the fact that $\epsilon$ can be made arbitrarily small.
This implies 
\begin{align*}
\lim_{T \rightarrow \infty} \frac{\Delta_{s_ij} \cdot \E_{\nu \A}[N_j(T')]}{\log T} 
    = \Omega\paren{\frac{1}{\Delta_{s_ij} ^2}}
        \,,
\end{align*}
As noted previously, the regret contributed due to $j$ being selected 
anytime in $[T']$ is at least $\Delta_{s_j  i}$.
Hence, we have that 
\begin{align*}
 \lim_{T \rightarrow \infty}\frac{\E_{\nu}[R_T(\A)]}{\log T} &\geq
 \lim_{T \rightarrow \infty}\sum_{i = 1}^{|S|} \sum_{j \in D_{s_i} } \frac{\Delta_{s_ij} \cdot \E_{\nu }[N_j(T_i)]}{\log T} 
 + \sum_{i = 1}^{|S|-1}  \frac{\Delta_{s_is_{i+1}} \cdot \E_{\nu }[N_{s_{i+1}}(T_i)]}{\log T} \\
 & = \Omega \left( \sum_{i = 1}^{|S|-1} \frac{1}{ \Delta_{s_is_{i+1}}} + \sum_{i = 1}^{|S|} \sum_{j \in D_{s_i}} \frac{1}{ \Delta_{s_ij}}\right)
    \,.
\end{align*}
This proves the statement of our lower bound. 
\end{proof}

\section{Extensions Based on Practical Considerations}
\label{app:extensions}

\subsection{Unknown Utilities}
\label{sec:practical-social-learning}
In this section, we shall discuss simple approaches one might adopt in practical scenarios to estimate user preferences in the event that they are unknown to the platform ahead of time. These approaches can be combined with our more technical algorithms presented in Sections ~\ref{sec:stochastic} and ~\ref{sec:adversarial} as a two-phase estimate-then-optimize overall algorithm, where there is a small burn-in period where we only aim to learn the users preferences, following which we minimize regret once these preferences have been learned with sufficient precision. A key assumption here is that the users' preferences are time invariant\footnote{more generally, changing very slowly with time as we can re-run our estimation subroutine to update the preference from time to time.}, and that their attention windows are stochastic and lazy, i.e. shorter attention window lengths are more frequent than longer ones (see Section ~\ref{sec:eps-greedy}). Supposing the users are deterministic in the selection behavior, i.e. they greedily select the highest utility item from within their attention window, then the estimation problem becomes quite straightforward -- one can simply simulate any sorting algorithm by placing the pair of items whose relative order of preferences is to be determined in the top 2 positions in the permutation, and then playing that same fixed permutation until one of these items is selected. This will necessarily occur when an attention window of length 2 realizes, the probability of which is $\Omega(1/n)$ due to the laziness assumption. Therefore, the answer to any pairwise comparison queried by the underlying sorting algorithm can be obtained in $O(n\log T)$ trials with polynomially large probability, and since most query-efficient sorting algorithms require just $O(n\log n)$ queries, this simulation can be performed in just $O(n^2\log n\log T)$ trials with high probability.

A more interesting selection behavior is the following, which captures the essence of ``social-learning'' where the users themselves are initially unaware of their true preferences, but are learning through the selections of their peers. Specifically, let us assume that there are some underlying true item utilities that are only approximately known to both the users as well as the platform through confidence intervals, i.e. for each item, both the users and the platform know a range that contains its true utility. The width of this range is assumed to be inversely proportional to some monotone function of the total number of selections of this item thus far. A natural choice of one such function would be the square root, owing to concentration of subgaussians. A practical motivation behind this modeling choice might be one where the users perceptions of the utilities themselves are being shaped by the selection patterns of their peers; each selection results in a new unbiased estimate of the underlying true utility which is visible as a public review or score, with the total number of selections providing newly arriving users with ``confidence'' of its true utility in the form of an interval of shrinking width around the mean of these utility estimates. When a newly arrived user is presented with these approximate utilities, he fixes for each item a perceived utility, which is a value chosen arbitrarily from the confidence interval of its true utility. Upon being presented a permutation, he then selects the item with highest perceived utility within his attention window, and for the selected item, adds his (unbiased) estimate of the true utility of the selected item to the public review. Under this modeling assumption, the following approach can be used to efficiently learn the underlying true item utilities. More precisely, since all our algorithms only require knowledge of the ordering  of item utilities, assuming the aforementioned model for selection behavior, it suffices to reach a state where the confidence intervals of item utilities are disjoint. In order to satisfy this condition, simply positioning the least-selected item whose utility interval overlaps with at least one other utility interval at the top of the displayed permutation suffices. Supposing the minimum separation in utilities is $\Delta^u_{\min}:= \min_{i,j\in [n]} |u_i-u_j|$, then it is straightforward to see that this approach would separate all utility intervals within $O( (\Delta^u_{\min})^{-2}n^2\log T)$ trials: due to the lazy assumption, a window of length 1 realizes every $\Omega(n\log T)$ trials with polynomially large probability, and when the number of selections of all items exceeds $(\Delta^u_{\min})^{-2}$, all utility intervals are necessarily disjoint.

\subsection{Delayed Feedback}
In this section we provide an extension of our results for a setting where the payoff feedback is delayed. This setting is motivated by applications where 
there might be delays in observing the longer-term payoff of current actions. 
We consider a simple setting of delayed feedback 
where the payoff for round $t$ is observed  
after a delay of $\tau_t$, i.e., the payoff for chosen item at time $t$ is observed at time $t+\tau_t$.
Note that the item selected by the user is revealed instantaneously to the algorithm, but the 
payoff for the selected item might be available after a delay.
We assume that $\tau_t$ is a integer random variable that is bounded between $0$ and $\tau_{\max}$.

\citep{joulani2013online} gave a black-box reduction for classical multi-armed bandits that 
converts any algorithm for the non-delayed setting 
to an algorithm for the delayed setting. 
In particular, \cite{joulani2013online} showed that delayed feedback increased the regret in a multiplicative way for adversarial bandit problems while in an additive way for stochastic bandit problems.  
Here, we follow the black-box reduction approach from \citep{joulani2013online} for handling delayed feedback.

We first consider the stochastic payoffs case from Section~\ref{sec:stochastic}. We use the QPM-D algorithm
from \cite{joulani2013online} which is a meta algorithm that simulates a base algorithm in a delayed feedback environment. This algorithm maintains a FIFO queue $Q_i$ for each arm $i \in [K]$. It stores the payoffs received for item $i$ from previous rounds in the queue $Q_i$
as soon as they become available. If the base algorithm 
wants to play arm $i$ and a payoff is available
in the queue $Q_i$ then it will provide the base algorithm with this payoff. Otherwise it will put the base algorithm on hold, and keep playing arm $i$ until a payoff becomes available. Using the same black-box reduction given in QPM-D gives us the following regret guarantee
for our stochastic setting with delayed feedback assuming that the attention window are stochastically generated under the lazy assumption in each round:
$\E[\textsc{Regret}(T)] \leq \E[\textsc{Regret}^{\text{BASE}}(T)] + n \cdot\tau_{\max}$
where BASE is any algorithm for the stochastic setting for our problem. One can also get a dependence on the 
expected delay in this setting with some additional work.

We now move to the adversarial payoffs case from Section~\ref{sec:adversarial}. We use the BOLD algorithm
from \cite{joulani2013online} which is also a meta algorithm that simulates a base adversarial bandit algorithm in a delayed feedback environment. 
This algorithm runs multiple instances of the base algorithm instead of a single instance. 
It spawns a new instance of the base algorithm and plays an arm according to this instance if all the existing instances are waiting for feedback. Once the feedback is available for an existing instance, it becomes active again and an arm can be played according to this instance. The algorithm bounds the number of spawned instances in terms of the parameter of the delay distribution. Using the same black-box reduction given in BOLD gives us the following regret guarantee
for our adversarial payoffs setting with delayed feedback:
$\E[\textsc{Regret}(T)] \leq (\tau_{\max}+1) \cdot \E[\textsc{Regret}^{\text{BASE}}(T/(\tau_{\max}+1))]  $
where BASE is any algorithm for the adversarial payoffs setting from Section~\ref{sec:adversarial}. 
Once again we can get a dependence on the expected delay with some additional work similar to \cite{joulani2013online}.

\end{document}